\documentclass[twoside,11pt]{article}

\usepackage{jmlr2e_mod}
\usepackage{amssymb}
\usepackage{amsfonts}
\usepackage{mathtools}
\mathtoolsset{showonlyrefs}
\usepackage{amsmath}
\usepackage{mathrsfs}
\usepackage{bm}
\usepackage{bbm}

\newcommand{\R}{\mathbb{R}}
\DeclarePairedDelimiter{\norm}{\lVert}{\rVert}

\newcommand{\floor}[1]{\lfloor #1 \rfloor}

\newcommand*\mcup{\mathbin{\mathpalette\mcupinn\relax}}
\newcommand*\mcupinn[2]{\vcenter{\hbox{$\mathsurround=0pt
  \ifx\displaystyle#1\textstyle\else#1\fi\bigcup$}}}

\newcommand*\mcapinn[2]{\vcenter{\hbox{$\mathsurround=0pt
  \ifx\displaystyle#1\textstyle\else#1\fi\bigcap$}}}
\DeclarePairedDelimiter{\abs}{\lvert}{\rvert}

\DeclarePairedDelimiter{\parr}{(}{)}
\DeclarePairedDelimiter{\parq}{[}{]}

\DeclarePairedDelimiter{\bra}{\lbrace}{\rbrace}

\DeclarePairedDelimiter{\ceil}{\lceil}{\rceil}
\DeclarePairedDelimiter{\prodscal}{\langle}{\rangle}

\DeclareMathOperator*{\st}{\,:\,}

\newcommand {\E} {{\mathbb{E}}}

\newcommand {\NN} {{\mathcal{N}}}

\newcommand {\C} {{\mathbb{C}}}

\renewcommand{\S}{{\mathbb{S}^{d-1}}}

\newcommand {\bx} {{\mathbf{x}}}
\newcommand {\by} {{\mathbf{y}}}
\newcommand {\ba} {{\mathbf{a}}}
\newcommand {\bb} {{\mathbf{b}}}
\newcommand {\bu} {{\mathbf{u}}}

\newcommand {\bw} {{\mathbf{w}}}

\newcommand {\bc} {{\mathbf{c}}}
\newcommand {\bg} {{\mathbf{g}}}
\newcommand {\bh} {{\mathbf{h}}}

\newcommand {\bv} {{\mathbf{v}}}
\newcommand {\bq} {{\mathbf{q}}}
\newcommand {\bp} {{\mathbf{p}}}

\newcommand {\bone} {{\mathbf{1}}}
\newcommand {\bzero} {{\mathbf{0}}}
\newcommand {\bfun} {{\mathbf{f}}}
\newcommand {\bA} {{\mathbf{A}}}
\newcommand {\bB} {{\mathbf{B}}}

\newcommand {\bD} {{\mathbf{D}}}
\newcommand {\bE} {{\mathbf{E}}}
\newcommand {\bU} {{\mathbf{U}}}
\newcommand {\bW} {{\mathbf{W}}}

\newcommand {\bI} {{\mathbf{I}}}

\newcommand {\bX} {{\mathbf{X}}}

\newcommand{\brho}{{\boldsymbol \rho}}
\newcommand{\balpha}{{\boldsymbol \alpha}}
\newcommand{\bbeta}{{\boldsymbol \beta}}
\newcommand{\bepsilon}{{\boldsymbol \epsilon}}

\newcommand{\bgamma}{{\boldsymbol \gamma}}

\newcommand{\bsigma}{{\boldsymbol \sigma}}

\newcommand{\bxi}{{\boldsymbol \xi}}

\newcommand{\bnu}{{\boldsymbol \nu}}

\newtheorem{assumption}{Assumption}

\ShortHeadings{Depth separation beyond radial functions}{Venturi, Jelassi, Ozuch, Bruna}
\firstpageno{1}

\begin{document}

\title{Depth separation beyond radial functions}

\author{\name Luca Venturi \email venturi@cims.nyu.edu \\
       \addr Courant Institute of Mathematical Sciences\\
       New York University\\
       New York, NY 10012, USA
       \AND
       \name Samy Jelassi \email sjelassi@princeton.edu\\
       \addr Department of Operations Research and Financial Engineering\\
       Princeton University\\
       Princeton, NJ 08540, USA
       \AND
       \name Tristan Ozuch \email ozuch@mit.edu\\
       \addr Department of Mathematics\\\
       Massachusetts Institute of Technology\\
       Cambridge, MA 02142, USA
       \AND
       \name Joan Bruna \email bruna@cims.nyu.edu \\
       \addr Courant Institute of Mathematical Sciences and Center for Data Science\\
       New York University\\
       New York, NY 10011, USA
       }

\maketitle

\begin{abstract}
High-dimensional depth separation results for neural networks show that certain functions can be efficiently approximated by two-hidden-layer networks but not by one-hidden-layer ones in high-dimensions $d$. Existing results of this type mainly focus on functions with an underlying radial or one-dimensional structure, which are usually not encountered in practice. The first contribution of this paper is to extend such results to a more general class of functions, namely functions with piece-wise oscillatory structure, by building on the proof strategy of \citep{eldan2016power}. We complement these results by showing that, if the domain radius and the rate of oscillation of the objective function are constant, then approximation by one-hidden-layer networks holds at a $\mathrm{poly}(d)$ rate for any fixed error threshold.  

A common theme in the proofs of depth-separation results is the fact that one-hidden-layer networks fail to approximate high-energy functions whose Fourier representation is spread in the domain. On the other hand, existing approximation results of a function by  one-hidden-layer neural networks rely on the function having a sparse Fourier representation. The choice of the domain also represents a source of gaps between upper and lower approximation bounds. Focusing on a fixed approximation domain, namely the sphere $\S$ in dimension $d$, we provide a characterization of both functions which are efficiently approximable by one-hidden-layer networks and of functions which are provably not, in terms of their Fourier expansion.
\end{abstract}

\begin{keywords}
Neural networks, Depth separation
\end{keywords}

\section{Introduction}

Learning in high-dimensions is a challenging task for computational, statistical and approximation reasons. Even in the classic supervised learning setup, current empirical successes of deep learning algorithms remain largely out of reach for existing theories, despite phenomenal recent progress.
Amongst the algorithmic aspects enabling this success, depth remains a major non-negotiable element. Depth in structured neural networks such as convolutional neural networks provides a multiscale processing of information, but more generally it defines an intricate function class with powerful approximation biases. 

Understanding the benefits of depth for approximating certain functions of interest represents a long-standing problem. The classic result of the universal approximation theorem ensures approximation by neural networks of any continuous function, but it focuses on \emph{shallow} (that is, one-hidden-layer) models and comes with possibly exponential (in the dimension) rates. The seminal work \citep{barron1993universal} provides dimension-free quadratic approximation rates by shallow networks under a condition of sparsity of the Fourier transform. 
Recent works \citep{eldan2016power,daniely2017depth} suggest that this property is essentially necessary in order to recover polynomial approximation rates, by constructing examples of deep networks which are spread in direction and away from zero in the frequency regime, and by showing that these function can not be efficiently approximated by a shallow counterpart. 
These depth-separation phenomena occur in the high-dimensional regime, where approximation by neural networks of standard Sobolev spaces is cursed (see e.g. \citep{maiorov2000near}). On the other hand, proofs of such high-dimensional depth-separation phenomena are currently limited to radial functions, that is of the form $f(\bx) = \varphi( \norm{\bA \bx + \bb}_2 )$. 

In this work we extend the results just cited, further cementing Barron's intuition. We describe rates of approximation by one-hidden-layer networks in terms of the number of units $N$ of the network, by looking at the Fourier representation of the function to be approximated. We consider two types of approximation rate, inspired by the work \citep{safran2019depth}: (i) the rate of approximation is polynomial in both the input dimension $d$ and the error estimation $\epsilon$, that is $N \simeq \mathrm{poly}(d,\epsilon^{-1})$ -- we refer to this rate of approximation as \emph{universal} approximation (ii) for any fixed error threshold $\epsilon$, the number of units $N$ needed for approximation of approximation depends at most polynomially on $d$, that is $N \simeq \mathrm{poly}(d)$ for any fixed error threshold $\epsilon$ -- we refer to this rate of approximation as \emph{fixed-threshold} approximation.  
We distinguish two fundamentally different regimes of approximation: relative to a heavy-tailed, unbounded data distribution, or relative to a concentrated distribution. Whereas the former captures the most general setup, the latter is motivated by practical machine learning applications. Our contributions are as follows.
\begin{itemize}

\item First, we consider a class of two-hidden-layer networks exhibiting piece-wise oscillatory behavior, namely functions of the form
$$
f_{r,\bw,\bv}:\bx\in\R^d \mapsto e^{2\pi i r\, \parr{\bv^T\bx \,+\, \bw^T\bx_+}}~.
$$
In section \ref{sec:lb}, we show that, under appropriately heavy-tailed data distributions, approximation at a rate $N\simeq \mathrm{poly}(d)$ cannot hold (unconditionally on the weights of the approximant network), as long as the rate of oscillations $r$ grows faster than $d$. On the other hand, $f_{r,\bw,\bv}$ can be universally approximated (that is, at a rate $\mathrm{poly}(d,\epsilon^{-1})$) by a two-hidden-layer network with any practical activation of choice. The proof of this result (Theorem \ref{theo:ds_informal}) extends the main idea introduced by the results of Eldan and Shamir \citep{eldan2016power} beyond the radial case.

\item In section \ref{sec:ub}, we show that the $\mathrm{poly}(d)$-oscillatory aspect and the heavy-tailed data distributions are necessary in the depth-separation result mentioned above. More specifically, we show that any deep network, with $O(1)$-bounded weights and $O(1)$-Lipschitz activation, can be fixed-threshold approximated by one-hidden-neural networks over a compact set of radius $O(1)$ (Theorem \ref{theo:approx_shallow}). This extends an equivalent result in \citep{safran2019depth}, from the class of radial functions to the one of deep neural networks with H\"{o}lder activations.

\item Aforementioned depth separation results consider functions whose Fourier representation is spread in high frequencies. On the other hand, universal approximation results often require the function to be approximated to be, in some sense, sparse in the Fourier domain. Unfortunately, there are currently many gaps between these two types of results, one of them being the definition of approximation domain. In order to reduce the gap between the two results above, we consider approximation on a fixed compact domain, namely the unit sphere $\S$, where Fourier analysis can be done using spherical harmonics. We individuate two conditions on the spherical harmonics  decomposition of a function $f \in C(\S)$. The first is a sparsity condition on the decomposition, which we show to be sufficient to prove universal approximation (that is, at a rate $N\simeq \mathrm{poly}(d,\epsilon^{-1})$) of $f$ by one-hidden-layer networks. The second is a high-energy spreadness condition on the spherical harmonics decomposition of $f$, which we show to imply that universal approximation of $f$ by one-hidden-layer networks cannot hold. This is the content of section \ref{sec:sh}, of which the main results are summarized in 
section \ref{sec:sh_main_results}.

\end{itemize}

\subsection{Related works}

There is a huge literature of approximation results for neural networks. Early approximation results provided upper and lower bounds on the approximation of some functional spaces such as Sobolev spaces \citep{maiorov2000near} or $L^p$ spaces \citep{pinkus1999approximation} by neural networks. For high input dimensions $d$, such results hold for functions with smoothness proportional to $d$, or require an approximation rate that scales as $N \sim \epsilon^{-d}$ (see e.g. \citep{petersen2020neural,guhring2020expressivity} for a review), where $N$ denotes the number of units of the network and $\epsilon$ the error threshold. 

In more recent years, quite a few works pointed out the benefits of deep networks versus their shallow counterparts from the point of view of approximation rates. For example, this has been shown for sawtooth function \citep{telgarsky2016benefits}, functions with positive curvature \citep{liang2016deep,yarotsky2017error,safran2017depth}, functions with a compositional structure \citep{poggio2017and}, piecewise smooth functions \citep{petersen2018optimal}, Gaussian mixture models \citep{jalali2019efficient}, polynomials \citep{rolnick2017power}, or model reduction models \citep{rim2020depth}. The result of \citep{telgarsky2016benefits} has been further generalized using a notion of periodicity \citep{chatziafratis2019depth}. It must be noticed that most of the cited works show depth separation that is independent of the dimension $d$ and that increases exponentially with the depth of the network. Another line of works \citep{eldan2016power,daniely2017depth,safran2019depth} on the other hand shows depth separation exponential in the dimension $d$, between networks with one and two hidden layers. This is the framework of this work. It was also shown recently that depth separation results between fixed depths greater than this are arguably difficult to prove \citep{vardi2020neural,vardi2021size}. 

This depth-width trade-off has been analyzed through different lens than approximation capabilities, such as classification capabilities \citep{malach2019deeper}, exact representability \citep{arora2016understanding}, Betti numbers \citep{bianchini2014complexity}, number of linear regions \citep{pascanu2013number,montufar2014number,raghu2017expressive,hanin2019complexity,hanin2019deep}, trajectory lengths \citep{raghu2017expressive}, globale curvature \citep{poole2016exponential} or topological entropy \citep{bu2020depth}. In essence, all these results state that networks expressivity improve exponentially as we increase the depth. Another related question is whether depth-separation holds from a learnability (therefore, not solely approximation) point of view as well \citep{malach2019deeper,malach2021connection}. In this work we focus on approximation and we consider the Fourier representation as a complexity measure. This is the approach followed by e.g. \citep{eldan2016power, daniely2017depth}, which construct examples of deep neural networks, whose Fourier energy is exponentially higher than those of shallow neural networks with a moderate number of units. 

On the other hand, sparsity of the Fourier transform has been used to show polynomial rates of approximation of functions by neural networks \citep{klusowski2018approximation,ongie2019function,bresler2020sharp}. In the last part of the paper, we show that an equivalent condition can be described in terms of spherical harmonics decomposition.

\section{Preliminaries}

\subsection{Neural networks}\label{sec:neural_nets}

For $L\geq 1$, we call an $L$-hidden-layer feed-forward neural network  a function 
\begin{equation}\label{eq:neural_network}
f : \bx \in \R^d \to \bx^{(L+1)}(\bx) \in \C^{d_{L+1}}~,
\end{equation}
where $\bx^{(L)}$ is defined by recursion by $\bx^{(0)}(\bx) = \bx$,
\begin{align*}
\bx^{(k)}(\bx) & = \sigma^{(k)}(\bA^{(k)} \bx^{(k-1)}(\bx) ) ~ \text{for } k \in [L] \quad\text{and }\quad
\bx^{(L+1)}(\bx)  = \bA^{(L+1)} \bx^{(L)}(\bx)~,
\end{align*}
where 
\begin{align*}
\bA^{(k)} & = \parq{\ba_{1}^{(k)} |\cdots | \ba_{d_k}^{(k)}}^T \in \R^{d_{k}\times d_{k-1}}\quad \text{ for $k \in [L]$,}  \\  
\bA^{(L+1)} & = \parq{\ba_{1}^{(L+1)} |\cdots | \ba_{d_{L+1}}^{(L+1)}}^T \in \C^{d_{L+1}\times d_L}
\end{align*}
(with $d_0 = d$) and $\sigma^{(k)} : \R^{d_{k}} \to \R^{d_{k}}$ are \emph{activation} functions, that is $\parr*{\sigma^{(k)}(\bx)}_i = \sigma_i^{(k)}(x_i)$ for some function $\sigma^{(k)}_i:\R\to\R$. A neural network is therefore a sequence of sums and compositions of \emph{ridge} functions, that is functions of the form $\bx\mapsto \sigma(\bw^T\bx)$. In the following, unless specified, we only consider neural networks (or, more simply, networks) as defined in \eqref{eq:neural_network}.  Most of the times we will deal with real-valued networks, that is $\bA^{(L+1)} \in \R^{d_{L+1}\times d_L}$. We say that a network has activation $\sigma$ if $\sigma^{(k)}_i(x) = \sigma(x + b_i^k)$ for some bias term $b_i^k \in \R$ for all $k,i$. We refer to the function
$$
\bx\in \R^{d_{k-1}} \mapsto \sigma^{(k)}(\bA^{(k)} \bx) \in\R^{d_k}
$$
as $k$-th \emph{hidden (or inner) layer} of \emph{width} $d_k$, for $k \in [L]$, while we refer to the linear function defined by $\bA^{(L+1)}$ as the last (or $L+1$-th) layer. We refer to the value $W(f) \doteq \max_{k\in[L]}d_k$ as width of the network $f$ and to the vectors $\ba^{(k)}_i$ as \emph{weights} (of the $k$-th layer), for all $k,i$. A basic complexity measure for neural network \eqref{eq:neural_network} is given by the total number of units, or \emph{size}:
\begin{equation}\label{eq:def:size}
N(f) \doteq \sum_{k=1}^L d_k ~.
\end{equation}
The number of layers $L(f) = L$ is also a relevant measure of complexity, which we refer to as \emph{depth}. Finally, in the following we sometimes require a control on the value of the weights; such controls are expressed in terms of norm $p$ of the weights, that is
\begin{equation}\label{eq:def:norm}
m_p(f) \doteq \max_{k,i}\norm{\ba_{k,i}}_p~,
\end{equation}
for some $p\in[1,\infty]$.

\subsection{Neural network approximation rates}

We measure the approximation error between two functions $f,g:\Omega\subseteq \R^d \to \C$ in terms of the $L^2(\mu)$ norm (with respect to a probability measure or density $\mu$) 
$$
\norm{f-g}_{\mu,2}^2 \doteq \int_\Omega \abs*{f(\bx)-g(\bx)}^2\,d\mu(\bx)~,
$$
or $L^\infty$ norm
$$
\norm{f-g}_{\Omega,\infty} \doteq \sup_{\bx\in\Omega}\abs*{f(\bx)-g(\bx)}~.
$$
Notice that a (uniform) $L^2$ lower bound implies a $L^\infty$ one, and viceversa for an upper bound.  The focus of this chapter is to establish upper and lower bounds for approximation of certain function classes by shallow neural networks, in high dimensions $d$. We distinguish two different approximation regimes of interest.
\begin{definition}
We say that a sequence $\bra*{f^{(d)}:\Omega_d\subseteq \R^d\to\C}_{d\geq 2}$ is \emph{universally approximable} by one-hidden-layer networks with activation $\sigma$ if it is approximable at a $\mathrm{poly}(d,\epsilon^{-1})$ rate; that is if there exists some constants $\alpha>0$ and $\beta>0$ such that it holds 
$$
\norm*{f^{(d)} -  f_N }_{\Omega_d,\infty} \leq \epsilon
$$
for some one-hidden-layer $f_N\in\mathcal{F}_N^\sigma$ satisfying $N + m_\infty(f_N) \leq \alpha\parr{d\epsilon^{-1}}^\beta$.
\end{definition} 
\begin{definition}
We say that $\bra*{f^{(d)}}_d$ is \emph{fixed-threshold approximable} if for any $\epsilon\in(0,1)$ it is $\epsilon$-approximable at a $\mathrm{poly}(d)$ rate; that is if for any $\epsilon > 0$  there exists some constants $\alpha>0$ and $\beta>0$ such that
for every $\epsilon>0$ it holds 
$$
\norm*{f^{(d)} -  f_N }_{\Omega_d,\infty} \leq \epsilon
$$
for some one-hidden-layer $f_N\in\mathcal{F}_N^\sigma$ satisfying $N + m_\infty(f_N) \leq \alpha d^\beta$.
\end{definition} These approximation schemes were introduced in \citep{safran2019depth}. To ensure significance of the approximation rates, in the following upper and lower bounds are stated for objective functions $f^{(d)}$ normalized such that $\norm{f^{(d)}}_2 \leq 1$ or $\norm{f^{(d)}}_\infty \leq 1$. 

\subsection{Activation assumptions} 

Finally, the results in the next sections generally hold for activations satisfying the following assumptions, which are satisfied by common activation such as the ReLU $\mathrm{ReLU}(x) = x_+$ or the sigmoid $\mathrm{sigmoid}(x) = \parr*{1+e^{-x}}^{-1}$ \citep{eldan2016power}. Most of the results can be easily generalized to hold under less strict conditions, but we take these assumptions for sake of simplicity.
\begin{assumption}\label{ass:activation}
Given an activation $\sigma:\R\to\R$, there exist constants $\iota_\sigma$ and $\nu_\sigma$ such that
\begin{enumerate}
\item it is $\iota_\sigma$-Lipschitz and $\sigma(0) \leq \iota_\sigma$;
\item for any $L$-Lipschitz function $f:\R\to\R$ constant outside of an interval $[-R,R]$ and any $\epsilon>0$ there exits $f_N \in \mathcal{F}_N^\sigma$ with $\norm{f - f_N}_\infty \leq \epsilon$ such that $N + w_\infty(f_N) \leq \nu_\sigma LR\epsilon^{-1}$.
\end{enumerate}
\end{assumption}

Notice that this assumption implies that, given a (deep) neural network $f$ with $\mathrm{poly}(d)$ weights and activations satisfying Assumption \ref{ass:activation}, then we are always able to replace the activations in $f$ by any other activation satisfying Assumption \ref{ass:activation}, by paying an at most polynomial cost. This is formalized in the following lemma.

\begin{lemma}\label{lemma:change_activation}
Let $\bra*{f^{(d)}: K_d\subset \R^d \to \C }_{d}$ be neural networks with activations satisfying Assumption \ref{ass:activation} and such that $N(f^{(d)}) + w_\infty(f^{(d)}) + \mathrm{diam}(K^{(d)}) \leq \mathrm{poly}(d)$; also let $\sigma$ be any activation function satisfying Assumption \ref{ass:activation}. Then the sequence $\bra*{f^{(d)}}_d$ is universally approximable by one-hidden-layer networks with activation $\sigma$.
\end{lemma}

\subsection{Notation}

We introduce notation we use throughout the rest of the paper.
 We denote scalar valued variables as lowercase non-bold; vector valued variables as lowercase bold; matrix and tensor valued variables and multivariate random variables (r.v.'s) as uppercase bold. Given a vector $\bv\in\mathbb{R}^d$, we denote its components as $v_k$; given a matrix $\bW\in\mathbb{R}^{n\times m}$, we denote its columns as $\bw_k$.
For a matrix $\bW$, we denote by $\norm{\bW}_{F,p}$ its entrywise $p$-norm, by $\norm{\bW}_{p,q}$ its $(p,q)$ operator norm (that is $\norm{\bW}_{p,q} = \max_{\norm{\by}_p=1}\norm{\bW\by}_q$) and by $\norm{\bW}_{p}$ its $p$ operator norm (that is $\norm{\bW}_p = \norm{\bW}_{p,p}$). We denote by $\mathbb{S}^{d-1}\subset \mathbb{R}^n$ the $(d-1)$-dimensional sphere $\bra*{\bx\in\mathbb{R}^d\st \norm{\bx}_2=1}$ and by $B_{r,p}^d$ the $\ell^p$ ball of radius $r$ in $\R^d$, that is $\bra*{\bx \in \R^d \st \norm{\bx}_p \leq r}$. We denote by $L^p(\Omega)$, $L^p(\mu)$, $L^p(\varphi)$ the spaces of functions $f:\Omega\to\mathbb{R}$ which are $p$-integrable with respect to the Lebesgue measure, the measure $\mu$ or the density $\varphi$, respectively. The respective norms (and scalar products for $p=2$) are denoted by $\norm{f}_{\zeta,p}$ ($\prodscal{f,g}_\zeta$) for $\zeta \in \bra{\Omega,\mu,\varphi}$; we simply write $\norm{f}_p$ when the measure is clear from the context. For a finite signed Borel measure $\mu$, we denote its total variation as $\norm{\mu}_1$. Finally, we denote by $\hat{f}$ or $\mathscr{F}(f)$ (resp. $\check{f}$ or $\mathscr{F}^*(f)$) the Fourier transform (resp. the inverse Fourier transform) of $f$ (meant in the following in the sense of tempered distributions).

\section{A depth separation example}\label{sec:lb}

Our starting point for the study of depth-separation is to consider a generic data distribution $\mu$ with adversarial properties against shallow approximations. 
In the seminal work \citep{eldan2016power}, Eldan and Shamir establish an 
unconditional (with no restrictions on the norms of the weights of the network) depth-separation result by considering a density $\mu$ in $\R^d$ with tails $\mu(\norm{\bx}_2) \simeq \norm{\bx}_2^{-(d+1)/2}$ 
and a radial function $f^{(d)}(\bx) = h_d(\norm{\bx}_2)$ with $h_d:\R \to \R$ a carefully chosen oscillating function with compact support. 
The proof in \citep{eldan2016power} reveals the limitations of shallow neural networks 
at approximating high-dimensional functions via a powerful harmonic analysis insight, that 
is particularly convenient in the setting of radial functions. 
In this section, we show that their proof strategy can be extended to include more diverse function classes, namely those arising naturally from ReLU networks. 
Specifically, we consider networks of the form
\begin{equation}
\label{eq:defig}
f_{r,\bw,\bv} : \bx\in\R^d \mapsto \sigma_r \parr*{ \bv^T\bx + \bw^T\bx_+} 
\end{equation}
where $\bx_+$ denotes the element-wise ReLU activation, $\bv,\, \bw \in \mathbb{R}^d$ and $\sigma_r(t) = e^{2 \pi i r t}$. We are thus considering a function which is piece-wise oscillatory, with constant envelope $|f_{r,\bw,\bv}(\bx)|=1$, and where the frequency of oscillations is controlled by $r$. The main result of this section can be summarized as follows.
\begin{theorem}[Informal]\label{theo:ds_informal}
Assume that $\norm{\bw}_2, \norm{\bv}_2 = \Theta(1)$ and that $r = \Theta(d^k)$ for some $k\geq 2$. Then there exists a (low-decay) product measure $\mu$ on $\R^d$ such that the function $f_{r,\bw,\bv}$ is universally approximable by two-hidden-layer networks but it is not fixed-threshold approximable by one-hidden-layer networks.
\end{theorem}

\subsection{The lower bound}\label{sec:lb_details}

Let $\psi \in L^2(\R) \cap L^1(\R)$ with $\|\psi\|_2 = 1$, and such that its Fourier transform $\hat{\psi}$ is compactly supported in $[-K,K]$, for some $K>0$. Assume also that 
\begin{equation}
\label{eq:lb_density_spreadness}
\norm{\psi}_1 < \sqrt{2/K}~.
\end{equation}
The condition ensure that the density $\psi$ is sufficiently spread away from zero (see Remark \ref{remark:psi_condition}).
Our first objective is to establish depth separation for the approximation of $f_{r,\bw,\bv}$ under the $L^2$ metric defined by the probability density $\varphi^2$, where $\varphi:\bx\in\R^d \mapsto \prod_{j=1}^d \psi(x_j)$.
\begin{theorem}\label{theo:lb}
Let $f^{(d)} = f_{r_d,\bw_d,\bv_d}$, for some $r_d \in \R$, $\bw_d,\bv_d \in \R^d$. For a fixed $\gamma > 0$, define
$$
\tau_d \doteq \sup_{S \subseteq [d] } \norm*{ \bv_d + \bw_{d,S}}_\infty \,, \quad \Omega_d \doteq \bra*{j\in [d] \st r_d\abs*{w_{d,j}} \geq \gamma d^2} \quad \text{and} \quad \eta_d \doteq \frac{\abs*{ \Omega_d }}{d} ~,
$$
where $\bw_{d,S} \in\R^d$ is defined by $w_{d,S,i} = w_i \mathbbm{1}\bra*{i \in S}$. 
Assume that
\begin{enumerate}
\item[(i)] oscillations grow polynomially, that is $\tau_d\cdot r_d = \Theta(d^k)$ for some constant $k > 0$;
\item[(ii)] the vectors $\bw_d$ are sufficiently spread, that is $\eta_d \geq \eta$ for some $\eta > 0$ independent of $d$;
\item[(iii)] the density $\varphi^2$ is sufficiently spread, i.e. 
 $2 K \norm{\psi}_1^2< 2^{2\eta} $.
\end{enumerate}
Then there exists a constant $\alpha \in (0,1)$ (independent of $d$) such that
\begin{align}
\label{eq:mainresres}
\inf_{f_N \in \mathcal{F}_N} \norm{f^{(d)} - f_N}^2_{\varphi^2,2} \geq 1 - N \cdot \alpha^d \cdot O\parr{ d^{k+1} }  ~.
\end{align}
Notice that this lower bound is unconditional on the weights of the neurons $m_\infty(f_N)$. 
\end{theorem}

The proof follows a similar strategy as in the work \citep{eldan2016power}. The approximation error can be expressed in the Fourier domain as 
$$
\norm{f_{r_d,\bw_d,\bv_d} - f_N}^2_{\varphi^2,2} = \| f_{r_d,\bw_d,\bv_d} \cdot \varphi - f_N \cdot \varphi\|_2^2 = \norm{ \hat{f}_{r_d,\bw_d,\bv_d} \ast \hat{\varphi} - \hat{f_N} \ast \hat{\varphi} }_2^2~.
$$
Thanks to the assumptions, the target function $f_{r_d,\bw_d,\bv_d}$ satisfies a key property, namely that its Fourier transform has its energy sufficiently spread in the high-frequencies, after the convolution by $\hat{\varphi}$. Such frequency spread is caused by the shattering of the first ReLU layer, which effectively creates $\Theta(2^{\eta d})$ different frequencies. The piece-wise structure arising from the ReLU can be handled in the Fourier domain by the Hilbert transform of the function $\psi$, which has sufficient decay thanks to the assumptions. Noticing that $\norm{\hat{f}_{r_d,\bw_d,\bv_d} \ast \hat{\varphi}}_2 = 1$, this is formalized in the following.

\begin{lemma}[Informal]\label{lemma:decay_two_layer_network}
It holds that
$$
\abs*{\parr*{\hat{f}_{r_d,\bw_d,\bv_d} \ast \hat{\varphi}}(\bxi)} \lesssim 2^{-\eta d } \norm{\varphi}_1 \norm{\bxi}_\infty^{-1} \qquad \text{for} \; \bxi \gtrsim \mathrm{poly}(d)~.
$$
\end{lemma}

On the other hand, since $\hat{\varphi}$ is compactly supported and the Fourier transform of a single-unit network is localised in a frequency ray, the Fourier transform of $f_{r_d,\bw_d,\bv_d} \cdot \varphi$ is localised in a union of $N$ tubes, of the form $T_\balpha = \mathrm{span}(\bra*{\balpha}) + [-K,K]^d$. 
This implies that 
$$
\inf_{f_N \in \mathcal{F}_N} \norm{f_{r_d,\bw_d,\bv_d} - f_N}^2_{\varphi^2,2} \geq \inf_{f_N \in \mathcal{T}_{(N)}} \norm{f_{r_d,\bw_d,\bv_d} - f_N}^2_{\varphi^2,2}
$$ 
where $\mathcal{T}_{(N)}$ denotes the set of $L^2$ functions such that their Fourier transform is supported on the union of $N$ tubes $T_{\balpha_1},\dots,T_{\balpha_N}$ as above, for some arbitrary $\balpha_1, \dots,\balpha_N \in \R^d$. Thanks to Plancherel's identity, and since $\norm{f_{r_d,\bw_d,\bv_d}}_{\varphi^2,2} = 1$, it further holds that $$
\inf_{f_N \in \mathcal{T}_{(N)}} \norm{f_{r_d,\bw_d,\bv_d} - f_N}^2_{\varphi^2,2} \geq 1 - N \cdot \sup_{\balpha\in\S} \norm*{ \mathbbm{1}_{T_\balpha}\cdot \parr*{ \hat{f}_{r_d,\bw_d,\bv_d}\ast\hat{\varphi} } }^2_{2}~,
$$
where $\mathbbm{1}_{T_\balpha}$ denotes the indicator function of $T_\balpha$. Lemma \ref{lemma:decay_two_layer_network} can then be used to show that such projections are exponentially (in $d$) small, which implies equation \eqref{mainresres}.
The detailed proof is deferred to section \ref{sec:proof:theo:lb}.



\begin{remark}
Theorem \ref{theo:lb} asks for two main conditions to hold. First, the magnitude of oscillations of the objective function (parametrised by $r_d$) must grow at least polynomially with $d$, similarly to the assumptions in the works \citep{eldan2016power} and \citep{daniely2017depth}. Second, the data distribution $\mu$ with density $\varphi^2$ should be heavy-tailed, in order for its Fourier transform to be sufficiently localised. 
When $r_d$ does not grow fast enough with $d$, the energy starts piling up at the low frequencies, creating an important roadblock to establish approximation lower-bounds, and leaving open the possibility of efficient shallow approximation. Similarly, when $\mu$ concentrates too quickly, the proof strategy also fails, due to the fact that in that case $\hat{\varphi}$ is too spread in the Fourier domain, creating full overlap of the energies.
\end{remark}

\begin{remark}\label{remark:psi_condition}
The admissibility condition \eqref{eq:lb_density_spreadness} is necessary  
since $\eta \leq 1$ by definition. Notice that 
$$
1 = \norm{\psi}_2^2 = \norm{\hat{\psi}}_2^2 \leq (2K) \norm{\hat{\psi}}_\infty^2 \leq (2K) \norm{\psi}_1^2
$$
and therefore condition \eqref{eq:lb_density_spreadness} can be considered as a requirement on the Fourier transform of $\psi$ not being too concentrated in the origin.
The choice $\psi(t) = \sqrt{3/2}\,\text{sinc}^2(\pi t)$ corresponds to $K=1$, $\| \psi \|_1 = \sqrt{3/2}$ 
and $\| \psi \|_2 = 1$, which verifies (\ref{eq:lb_density_spreadness}). In that case, from condition (ii) we need
$\eta > \frac{\log_2 3}{2}\approx 0.79~.$  
However, the choice $\psi(t) = C\text{sinc}(\pi t)$ (the equivalent separable version of the of density considered in \citep{eldan2016power}) is not admissible, since $\psi$ is not in $L^1$.  
The lower bound is optimized by finding compactly supported windows with an optimal $L^1$ to $L^2$ ratio of their Fourier transforms. 
\end{remark}



\begin{remark}
The theorem considers a separable ReLU transform $\bx \mapsto \bx_+$, combined with a 
separable data distribution $\mu$ with density $\varphi^2$. One could expect a similar 
lower bound to apply in the more general case of a layer of the form $\bx \mapsto (\mathbf{U} \bx + \mathbf{b})_+$, $\mathbf{U}\in \R^{d' \times d}, \mathbf{b} \in \R^{d'}$. Such general case replaces the Hilbert transform of $\psi$ with the Fourier transform of indicators of convex polytopes, which has been used in the context of ReLU networks to characterize spectral properties \citep{rahaman2019spectral}. 
\end{remark}

\begin{example}\label{example:ds_explicit}
We give an explicit example of a family of function $\bra*{f^{(d)} : \R^d\to\R}$ which satisfy the assumptions of Theorem \ref{theo:lb}. Consider the functions
$$
f^{(d)}(\bx) = \exp \parr*{ 2\pi i d^2 \sum_{k=1}^d  \max\bra*{0, x_k} }\;.
$$
Then, if $\mu_d$ is the product probability measure defined by the density in Remark \ref{remark:psi_condition}, that is
$$
\mu_d(d\bx) = \prod_{k=1}^d\parq*{ \frac{3}{2} \mathrm{sinc}^4(\pi x_k)\,dx_k}\;,
$$
then it holds that
$$
\inf_{f_N \in \mathcal{F}_N } \norm*{f_N(\bx) - f^{(d)}(\bx) }_{\mu_d,2}^2 \geq 1 - 1300N \cdot d^2 \cdot \parr*{0.75}^d~.
$$
For example, this implies that
$$
\inf_{f_N \in \mathcal{F}_N } \norm*{f_N(\bx) - f^{(d)}(\bx) }_{\mu_d,2} \geq \frac{1}{2}
$$
unless
$$
N \geq  \frac{1.3^d}{10^4 d^{3}} \;.
$$
The numbers are obtained by explicitly tracking the constant in the proof of Theorem \ref{theo:lb} (see section \ref{sec:proof:theo:lb} for more details). Finally, notice that the functions $f^{(d)}$ are not radial. Indeed, they show different behaviour over each orthant, thanks to the ReLU layer. On the other hand, radial functions would behave equally over any orthant. A similar reasoning would generally hold for radiality in certain directions of the input.
\end{example}


\subsection{The upper bound}

According to the definition of neural networks we gave in section \ref{sec:neural_nets}, the function $f_{r,\bw,\bv}$ is naturally a two-hidden-layer neural network. Although, while there are cases of sinusoidal activations being used in practice, activations such as ReLU or sigmoid are more relevant to practical applications. The following theorem, proved in section \ref{sec:proof:theo:depth-sep}, shows that we can efficiently represent the function $f_{r,\bw,\bv}$ in the hypothesis of the Theorem \ref{theo:lb} as a two-hidden-layer neural network with fixed activation, such as the ReLU or the sigmoid. The main technical difference with Lemma \ref{lemma:change_activation} is that the result is proved for approximation w.r.t. the probability measure with density $\varphi^2$ introduced above.

\begin{theorem}\label{theo:depth-sep}
Let $\sigma$ be an activation satisfying Assumption \ref{ass:activation}. Assume that there exists a constant $k\geq 1$ such that  $m_\infty(f_{r_d,\bv_d,\bw_d}) \leq O( d^k)$ and assume that $\psi$ is such that $|\psi(x)| = O(|x|^{-1})$.
Then, for every $\epsilon >0$, there exists $f_N \in \mathcal{F}_N^\sigma$ with
$$
N + m_\infty(f_N) \leq O\parr*{d^{2(1+k)} \epsilon^{-3/2}}
\quad
\text{such that}
\quad
\norm{f_N - f_{r_d,\bw_d,\bv_d} }_{\varphi^2,2}^2 \leq  \epsilon~.
$$
\end{theorem}
Theorems \ref{theo:lb} and  \ref{theo:depth-sep} therefore estabilish a depth separation result. If $f^{(d)} = f_{r_d,\bw_d,\bv_d}$ are defined  with $r_d,\bw_d,\bv_d$ satisfying the assumptions of both theorems (that is, they satisfy assumptions (i)-(ii)-(iii) of Theorem \ref{theo:lb} with $\tau_d\cdot r_d  =\Theta(d^k)$), then Theorem \ref{theo:lb} says that $\bra*{f^{(d)}}_{d}$ is not fixed-threshold approximable by one-hidden-layer networks, while Theorem \ref{theo:depth-sep} says that the sequence is universally approximable by two-hidden-layer networks with a fixed activation satisfying Assumption \ref{ass:activation}. For example, the family of functions considered in Example  \ref{example:ds_explicit} satisfies such assumptions.

We thus identify two key aspects responsible for such depth separation: heavy-tailed data and oscillations growing with dimension. In the next sections we want to understand how necessary these two conditions are. 
The next section shows that if these two condition do not hold anymore, then a lower bound such as the one in Theorem \ref{theo:lb} is not achievable; more specifically we show that the objective function is fixed-threshold approximable by one-hidden-layer networks. 

\section{Approximation of deep networks by shallow ones}\label{sec:ub}

\label{sec:poly(d)ub}

In this section, we show that any deep neural network $f$ (which include the target functions considered in the previous section) can be approximated by shallow ones at a rate which is polynomial in $d$, as long as the rate of oscillation in the inner layers of $f$ is constant in $d$ and the metric is concentrated in a ball of constant radius. We start by reporting the result in a general form for two-hidden-layer networks and we discuss some consequences and extensions afterwards. 

Consider a family of two-hidden-layers neural network $\bra{f^{(d)}:K_d \subset \R^d \to \C}$ of the form 
\begin{equation}\label{eq:3_layer_net_for_poly(d)}
f^{(d)}: \bx\in\R^d \mapsto \bgamma_d^T \bg\parr*{\bW_d^T \bh\parr*{\bU_d^T\bx}} \in \C  ~,  
\end{equation}
where $\bh = \bh^{(d)}:\R^{p_d} \to \R^{p_d}$ and $\bg = \bg^{(d)}:\R^{o_d} \to \R^{o_d}$ are, respectively, component-wise $1$-Lipschitz and $(1,\alpha)$-Holder\footnote{We say that a function $g:\R\to\R$ is $(1,\alpha)$-Holder if it holds that $\abs*{g(x) - g(y)} \leq \abs*{x-y}^\alpha$ for all $x,y\in\R$.} activation functions, and $\bU_d
\in\R^{d\times p_d}$, $\bW_d
\in \R^{p_d\times o_d}$, $\bgamma_d\in\C^{o_d}$. We wish to  approximate $f^{(d)}$ by one-hidden-layer neural networks with a given activation.

\begin{theorem}\label{theo:approx_shallow}
Assume that $\mathrm{diam}(K_d) = O(1)$ 
and that the networks $f^{(d)}$ have $\ell^1$ bounded weights, that is $m_1(f^{(d)}) = O(1)$. 
Then, for every activation $\sigma$ satisfying Assumption \ref{ass:activation}.2 and every $\epsilon \in (0,1)$ it holds that
\begin{equation}\label{eq:inf_generic_informal}
\inf_{f_N^\sigma \in \mathcal{F}_N^\sigma} \norm{f^{(d)} - f_N^\sigma}_{K,\infty} \leq \epsilon \quad \text{for some } N \leq \exp\parr*{O\parr*{\epsilon^{-1-2/\alpha} \log \parr*{p_d/\epsilon} } }~.
\end{equation}
Moreover, it is possible to choose $f^\sigma_N$ attaining \eqref{eq:inf_generic_informal} with $m_\infty\parr*{f^\sigma_N}$ satisfying a bound similar to the one on $N$, for example $m_\infty\parr*{f^\sigma_N} \leq (1+N^2)$. 
\end{theorem}

The proof is constructive and based on the following observation. Consider the case where $o_d = 1$, $\gamma_d = 1$, $p_d = p$ and $g(x) = x^r$ some positive integer $r$. If $h_k(x) = e^{ix}$ for all $k \in [p]$, then the function $f = f^{(d)}$ at \eqref{eq:3_layer_net_for_poly(d)} has form
$$
f(\bx) = \parr*{ \sum_{k=1}^N w_k e^{i \bu_k^T \bx} }^r
$$
for some $w\in \R^N$, $\bu_k\in \R^d$, where $N=p$.
By expanding the power we can write
$$
f(\bx) = \sum_{j_1 + \cdots + j_N = r} \binom{r}{j_1 \cdots j_N} \parr*{w_1^{j_1}\cdots w_N^{j_N}} e^{i\parr*{\sum_{h=1}^N j_h \bw_h}^T\bx} ~,
$$
that is a formulation of $f$ as a one-hidden-layer network with activation $\sigma_1(t) = e^{2\pi i t}$ (in the following we refer to this type of networks as \emph{shallow Fourier networks}) and a number of units that scales as $N^r$. Since both polynomials and trigonometric polynomials are universal approximators, with well known convergence rates, in the general case one can proceed as follows. Each of the non-linearities applied to the first hidden layer can be approximated by a trigonometric polynomial at a polynomial rate on the interval of interest. Similarly, every non-linearity applied to the second hidden layer can be approximated by a polynomial at a linear (in the degree of the polynomial) rate on the interval of interest. Assuming for simplicity that both rates behave as $\epsilon^{-1}$, where $\epsilon >0$ denotes the approximation error, the composition of the two approximation following the structure of the target network results in a shallow Fourier network (that is with activation $\sigma_1(t) = e^{2\pi i t}$) whose size $N$ behaves, roughly speaking, as 
$$
N \simeq \Theta\parr*{p \epsilon^{-2}}^{\epsilon^{-1}} ~.
$$
Moreover, it is also possible to control the value of the coefficients appearing in the final approximation. With this, we can approximate each summand in the shallow Fourier network by a one-hidden-layer network with activation $\sigma$ with a controlled number of units, thanks to Assumption \ref{ass:activation}.2. A more detailed statement and a formal proof are reported in appendix \ref{app:ub}.

In essence, in the Theorem \ref{theo:approx_shallow}, we show that it is possible to approximate a two-hidden-layer neural network with $\mathrm{constant}(d)$ oscillations at a $\mathrm{poly}(d)$ rate over a compact set of $\mathrm{constant}(d)$ radius. On the other hand, it easy to show that it is also possible to obtain approximation at a $\mathrm{poly}(\epsilon^{-1})$ rate (see section \ref{sec:fixed_d_app}), for fixed $d$. Finally, existing results in the literature (see  \citep{safran2019depth}) show that universal approximation is not possible, the counterexample being essentially a radial function. 

Interestingly, the upper bound in Theorem \ref{theo:approx_shallow} does not depend on the number of units in the second layer of the objective function. This parameter is \emph{hidden} in the control we impose on the $\ell^1$ norm of the objective weights. 
The proof technique of this upper bound highlights how the difficulty of approximating at $\mathrm{poly}(d,\epsilon^{-1})$ rate stems from the high-energy of the second layer, which requires the shallow network used for approximation to have a (potentially) exponential (in $d$) number of directions. Notice that the lower bound in Theorem \ref{theo:lb} actually tells that the function is not fixed-threshold approximable.
High oscillations in the lower bound \eqref{eq:mainresres} essentially ensure that an exponential (in $d$) number of neurons are necessary. An open question is then whether a low-decaying measure is, in general, necessary for such a result to hold.

Expanding on the proof technique above, it is possible to extend the result of Theorem \ref{theo:approx_shallow} to approximation of $L$-hidden-layers networks by shallow ones, which gives a rate scaling as $\exp\parr{O\parr{\epsilon^{-L} \log \parr*{p/\epsilon} } }$. 

\begin{theorem}\label{theo:approx_deep_shallow}
Let $f^{(d)}$ as in  \eqref{eq:neural_network}, with $O(1)$-Lipschitz activations, first hidden layer width $d_1 = p_d$, depth $L_d = L$ and bounded weights, that is $m_1(f^{(d)}) = O(1)$. Then  for every $\epsilon>0$ there exists a shallow Fourier network $f_N \in \mathcal{F}_N^\sigma$ with 
$$
N \leq \parr*{ p_d \cdot O \parr*{1 + \frac{1}{\epsilon^2}}  }^{O(L) \parr*{ 1 + \frac{1}{\epsilon}}^{L-1}} \quad \text{such that}\quad
\norm*{f^{(d)} - f_N}_{B_{1,\infty}^d,\infty} \leq \epsilon  ~.
$$
\end{theorem}

See section \ref{app:multi-layer-ub-polyd} for a formal statement and its proof. While it has been shown that generic $O(1)$-Lipschitz function can not be (computably) represented by neural networks with $N \simeq \mathrm{poly}(d)$ units \citep{vardi2021size}, an interesting related follow-up conjecture is whether our result can be generalized to any generic $O(1)$-Lipschitz function which is $\mathrm{poly}(d)$-computable. Notice that this is dependent on the choice of the uniform norm to measure the approximation error. For example, it has been shown that a rate $N \simeq \mathrm{poly}(d)$ is achievable for approximation in the $L^2$ norm with the uniform measure \citep{hsu2021approximation}. 

Finally, notice that the approximation rate shown in Theorem \ref{theo:approx_shallow} and Theorem \ref{theo:approx_deep_shallow} are actually polynomial in the size $p_d$ of the first hidden layer of $f^{(d)}$ rather than in the input dimension $d$. Although, up to choosing a worse (yet constant) exponent in $\epsilon$, we can replace $p_d$ by $d$ in the statement, by considering the function as a $(L+1)$-hidden-layer network, where the first layer is the identity. 

\subsection{Two cases of interest}

Theorem \ref{theo:approx_shallow} allows to recover, for any fixed threshold $\epsilon > 0$, a $\mathrm{poly}(d)$ rate for the approximation of $f_{r,\bw,\bv}$ by one-hidden-layer networks and it can be seen as a generalization of Theorem 1 in \citep{safran2019depth}. This is the content of the following corollaries.


\begin{corollary}[Radial functions]
Let $f^{(d)}(\bx) = \varphi_d(\norm{\bx}_2)$, where $\varphi_d:[-1,1]\to\R$ are $1$-Lipschitz, and $K_d = B_{1,2}^d$. Then, for any $\epsilon \in (0,1)$ it holds that
$$
\inf_{f_N^\sigma \in\mathcal{F}_N^\sigma} \norm*{f^\sigma_N - f^{(d)}}_{K_d,\infty} \leq \epsilon \quad \text{for some } N \leq \exp\parr*{O\parr*{\epsilon^{-5} \log \parr*{d/\epsilon} } }~.
$$
Moreover, $f^\sigma_N$ can be chosen so that $m_\infty(f^\sigma_N)\leq \exp\parr*{O\parr*{\epsilon^{-5} \log \parr*{d/\epsilon} } }$. 
\end{corollary}

Consider the functions $f^{(d)} : \bx\in\R^d \mapsto e^{i\bw_d^T\parr*{\bU_d\bx}_+}$ for some $\bw_d\in\R^{p_d}$, $\bU_d \in \R^{p_d\times d}$. This is a more general version of the function $f_{r,\bw,\bv}$ considered in section \ref{sec:lb}. If the weights are bounded, that is $m_1(f^{(d)}) = O(1)$, then Theorem \ref{theo:approx_shallow} implies the following. 


\begin{corollary}[Shallow approximation of \eqref{eq:defig}]\label{cor:ub_poly(d)}
If $r_d = O(1)$ and $K_d= B_{r_d,2}^d$, for any $\epsilon \in (0,1)$  it holds that
$$
\inf_{f_N^\sigma \in\mathcal{F}_N^\sigma} D_\infty\norm{f^\sigma_N - f^{(d)} }_{K_d,\infty} \leq \epsilon \quad \text{for some } N \leq \exp\parr*{O\parr*{\epsilon^{-2} \log \parr*{p_d/\epsilon} } }~.
$$
Moreover, $f^\sigma_N$ can be chosen so that $m_\infty(f^\sigma_N)\leq \exp\parr*{O\parr*{\epsilon^{-2} \log \parr*{p_d/\epsilon} } }$. 
\end{corollary}

Although the result of Corollary \ref{cor:ub_poly(d)} is established for approximation in the uniform norm over the unit ball, it is not difficult  to extend it to a result in $L^2$ over a measure that concentrated over a compact set of constant (in $d$) radius, such as a normalized Gaussian. A formal statement of this fact, along with the proof, is reported in section \ref{proofs:special_cases}. Compared with the result of section \ref{sec:lb}, Corollary \ref{cor:ub_poly(d)} implies the following. The function $f_{\bw,\bU}$ can be approximated, at a $\mathrm{poly}(d)$ rate over a compact set of constant radius if its weights have constant norm. On the other hand, if  the norm of the weights grows polynomially in $d$, then approximation at a $\mathrm{poly}(d)$ rate is not possible, under a polynomially slow decaying measure. An open question is whether approximation at a $\mathrm{poly}(d)$ rate is possible if only one of these two conditions hold.

\section{Approximation by shallow networks: a spherical harmonics analysis}\label{sec:sh}

As already discussed, difficulties in approximating functions in high dimension by shallow networks appear when the function has a Fourier transform spread in a (exponential) number of directions in (polynomial) high energy. On the other hand, the presence of only one of these two conditions is not enough to prevent efficient approximability. While the previous results highlight this, the lower bound presented in Theorem \ref{theo:lb} applies to a specific choice of error measure, with  (polynomially) slowly decaying tails.  

In this section, we aim to disentagle the role of the measure tail and understand how the Fourier representation can tell whether a function is efficiently approximable by a one-hidden-layer network or not. In particular, we focus on approximation results for functions defined over the $(d-1)$-dimensional sphere $\S$, 
for which a rich literature of Fourier analysis is available. 

First, we give a sufficient condition on the target function in terms of its spherical harmonics decomposition to be not efficiently approximable by shallow one-hidden-layer networks. This condition captures a slowly decaying and sufficiently spread spherical harmonic expansion.
We also show that certain symmetry properties imply this condition.  On the other hand, one may ask if a reverse statement holds. In this direction, building on existing theory, 
we provide a sufficient condition for approximation by one-hidden-layer networks. 

\subsection{Spherical harmonics decomposition}

Let $d \geq 2$ and $S^{d-1}$ ($S$ when the dimension is clear from the context) be the uniform measure over $\S$. 
The spherical harmonics are a particular orthonormal basis for $L^2(S)$. They consists of 
$$
\mcup_{k = 0}^\infty\, \mathrm{span}\parr*{ \bra*{Y_{k,i}^d}_{i = 1}^{N_k^d} } = \mcup_{k = 0}^\infty H_k^d
$$
where $Y^d_{k,i}$ is a restriction to $\S$ of an homogeneous harmonic polynomial of degree $k$. The projection operator over $H_k^d$ is given by
$$
\mathcal{P}_k^d: f\in L^2(S) \mapsto f_k \doteq \sum_{i=1}^{N_k^d} \prodscal{f, Y^d_{k,i}}Y^d_{k,i}~.
$$
Similarly, $\mathcal{P}_I$ denotes the operator $\oplus_{i\in I} \mathcal{P}_i^d$, for any $I\subseteq \mathbb{N}$. The function $f_k$ is referred to as the degree $k$ spherical harmonic component of the function $f$.
Since the spherical harmonic form an orthonormal basis of $L^2_{S}$, it holds that $f = \sum_{k= 0}^\infty f_k$ and $\norm{f}_{2}^2 = \sum_{k= 0}^\infty \norm{f_k}_{2}^2$ for every $f \in L^2(S)$, where $\norm{\cdot}_2$ denotes the norm in $L^2(S)$. As spherical harmonics decomposition can be seen as a generalization of Fourier series to dimensions $d\geq 3$, in the following we refer to the spherical harmonics decomposition of a function as its Fourier representation, interchangeably. 
The operator $\mathcal{P}_k$ can be associated with a kernel given by
$$
\sum_{i=1}^{N_k^d} Y_{k,i}^d(\bx) \overline{Y_{k,i}^d(\by)}  = 
N_k^d P_{k}^d\parr*{\bx^T\by}  
$$
where
$$
N_k^d = \frac{(2k + d - 2)(k+d-3)!}{k!(d-2)!} = \Theta\parr*{\sqrt{\frac{k+d}{kd}} \frac{(k+d)^{k+d}}{k^k d^d} \frac{d^2}{(k+d)^2}}
$$
is the dimension of $H_k^d$
and $P^d_k$ is the $((d-2)/2)$-Gegenbauer polynomial defined as 
$$
P^d_k(x) = k!\,\Gamma\parr*{\frac{d-1}{2}}\sum_{j=0}^{\floor{k/2}} (-1)^j\frac{(1-x^2)^jx^{k-2j}}{4^j j!(k-2j)!\Gamma\parr*{j + \frac{d-1}{2}}} ~.
$$
Let $\omega_d$ be  the Lebesgue area of the sphere:
$$
\omega_d = \omega_{d-1}  \frac{\sqrt{\pi}\,\Gamma\parr*{\frac{d-1}{2}}}{\Gamma\parr*{\frac{d}{2}}} = \frac{2\pi^{d/2}}{\Gamma\parr*{\frac{d}{2}}} = \Theta\parr*{ \frac{(2\pi e)^{d/2}}{d^{d/2 - 1/2}}} = \Theta\parr*{ \sqrt{d}\parr*{\frac{2\pi e}{d}}^{d/2}}~.
$$
The polynomials $\bra{(N_k^d)^{1/2}P_k^d\,}_{k\geq 0}$ form a basis of orthonormal polynomials for $L^2(\mu_d)$, where $\mu_d$ is the probability measure on $[-1,1]$ defined by
$$
d\mu_d(t) = \alpha_d(1-t^2)^{(d-3)/2}\,dt~,
$$
where $\alpha_d = \omega_{d-1} / \omega_d = \Theta(\sqrt{d})$.
Notice that, given a function $f \in L^2(S)$, it holds
$$
f_k(\bx) = N_k^d\int_\S f(\by) P_{k}^d\parr*{\bx^T\by}\,dS(\by)~.
$$
Moreover, if the function $f$ only depends on a linear projection of the input, the Funk-Hecke formula holds. 
\begin{theorem}[Funk-Hecke formula]
For every $\sigma:[-1,1] \to \C$ such that $\bx \in \S \mapsto \sigma(x_1)$ is in $L^2(S)$, and for every $\bw \in \S$, it holds that
$$
\int_\S \sigma(\bw^T\bx)P_k^d(\bxi^T\bx)\,dS(\bx) = \lambda_k P_k^d(\bxi^T\bw)
$$
where $\lambda_k = \prodscal{\sigma,P_k^d}_{\mu_d}$.
\end{theorem}
Functions of the form
$$
\bx\in\S \mapsto \alpha P_k^d(\bw^T\bx)
$$
for some $\alpha \in\R$ and $\bw\in\S$, are called zonal harmonics. By the Funk-Hecke formula it follows that
$$
\int_\S P_k^d(\bw^T \bx)P_k^d(\bv^T \bx)\,d{S}(\bx) = \parr{N_k^d}^{-1} P_k^d(\bw^T\bv)
$$
for any $\bw,\bv \in \S$. This implies that $H_k^d$ has an RKHS structure with kernel $K$ given by
$$
K(\bv,\bw) \doteq N_k^dP_k^d(\bv^T\bw) ~.
$$
In particular, zonal harmonics actually span $H_k^d$. Moreover, it can be shown that there exists $\bw_1, \dots, \bw_{N_k^d} \in \S$ such that $H_k^d = \mathrm{span}\parr{\bra*{P_k^d(\bw_i^T \cdot)}_{i=1}^{N_k^d}}$ \citep[Theorem 4.13]{efthimiou2014spherical}. For these facts and more details about spherical harmonics we refer to the books \citep{atkinson2012spherical, dai2013approximation}.

\subsection{Concentration and spreadness in \texorpdfstring{$H_k^d$}{H(k,d)} and main results}\label{sec:sh_main_results}

Intuitively, one can say function $f \in C(\S)$ is concentrated over $\S$ if there is an area $\Omega \subset \S$ such that the mass of $f$ is concentrated over $\Omega$. On the other hand one could say that $f$ is spread if it assumes non-negligible values uniformly over the sphere. The spreadness/concentration of the function $f$ can be quantified by looking at ratios of the type
$$
\ell_{q,p}(f) \doteq \frac{\norm{f}_q}{\norm{f}_p}
$$
for $1 \leq p < q \leq \infty$. Since the norms above are with respect to a probability measure, it holds that $\ell_{q,p}\geq 1$. Intuitively, the closest this ratio is to $1$, the more spread is the function. On the other hand, the largest this ratio, the more concentrated the function is. Consider the case of a function $f_k \in H_k^d$. Then, it holds that
$$
\ell_{\infty, 2}(f_k) \leq \sqrt{N_k^d}
$$
The equality is attained for functions of the type $f_k(\bx) = \alpha P_k^d(\bw^T\bx)$ for some $\alpha\in \mathbb{C}$ and $\bw\in\S$, i.e. zonal harmonics. In this sense, zonal harmonics could be considered as the most concentrated functions in $H_k^d$. A similar inequality can be shown for the quantity $\ell_{2,1}$: it holds that
\begin{equation}\label{eq:ell_2_1_bound}
\ell_{2,1}(f_k) \leq \sqrt{N_k^d}
\end{equation}
for $f_k\in H_k^d$. Nevertheless, in this case, zonal harmonics do not attain equality; the inequality is actually not tight; a more detail discussion on this quantity is reported in section \ref{sec:sh:concentrated}. 

Thanks to the Funk-Hecke formula, it holds that a one-hidden-layer $f_N \in \mathcal{F}_N$, with hidden layer weights given by $\bw_1,\dots,\bw_N$, satisfies
$$
\mathcal{P}_k^d f_N = \sum_{j=1}^d \alpha_j P_k^d(\bw^T_j \bx) 
$$
for some $\balpha \in \C^N$. In other words, its Fourier representation is concentrated along $N$ directions. According to the remarks above, this implies that if the width $N$ is relatively small, the Fourier components of the neural network $f_N$ are relatively concentrated in space. One would then expect that such concentration can be used to determine whether a function can be approximated efficiently by a one-hidden-layer neural network or not. In the next sections, we show that this is indeed the case. 
Let $f\in C(\S)$; assuming that $\norm{f_k^{(d)}}_2 \simeq \mathrm{poly}(d,k^{-1})$, the results can be informally summarized as follows:
\begin{itemize}
\item If the spherical components of $f$ are (exponentially) spread in $\ell_{\infty,2}$ sense, that is, for example,
$$
\ell_{\infty,2}(f_k) \lesssim \epsilon^k \cdot \sqrt{N_k^d} = \epsilon^k \cdot \sup_{g\in H_k^d}\ell_{\infty,2}(g) \quad \text{for some } \epsilon \in (0,1)
$$
then $f$ is provably not universally approximable  by one-hidden-layer networks.
\item If the spherical components of $f$ are (polynomially) concentrated in $\ell_{2,1}$ sense, that is, for example,
$$
\ell_{2,1}(f_k) \gtrsim
\mathrm{poly}(d^{-1},k^{-1}) \sqrt{N_k^d}
$$
then $f$ is universally approximable  by one-hidden-layer networks.
\end{itemize}
Notice that, on the other hand, if $\norm{f_k}_2$ decreases exponentially fast then universal approximation follows, and similarly  if $\norm{f_k}_2$ decreases exponentially slowly then universal approximation can not hold.
The first of the two conditions above expresses concentration of the Fourier decomposition, while the second expresses spreadness of the same. We notice at least two gaps between the two conditions. The first one is the expression of the concentration phenomena: one is with respect to $\ell_{\infty,2}$, while the other one is with respect to $\ell_{2,1}$. Second, the two regimes above do not include many other possible ones. For example, we suspect the existence of a regime which prevents universal approximability but allows for fixed-threshold one, a topic worth of future study. These results are properly formalized, stated and discussed in section \ref{sec:sh:spread} and section \ref{sec:sh:concentrated}, respectively. 

\subsection{Inapproximability of functions with spread Fourier representation}\label{sec:sh:spread}

As discussed above, one-hidden-layer functions have a \emph{zonal} structure. In more detail, if $h(\bx) = \sigma(\bw^T\bx + b)$ for some $\bw\in\S$ and $b \in \R$, then it is easy to see that
$$
h_k(\bx) = s_k \norm{h_k}_2 \sqrt{N_k^d} P_k^d(\bw^T\bx)
$$
with $s_k \in \bra{\pm 1}$. In particular, it follows that $\norm{h_k}_\infty = \abs*{h_k\parr{\pm \bw}} = \parr{N_k^d}^{1/2} \norm{h_k}_2$.
This can be interpreted by saying that the Fourier components of single neurons are most concentrated (along the neuron direction) in space. Therefore, it is natural to expect that functions with spread Fourier decomposition are difficult to approximate by neural networks. The proposition below formalizes this fact. The proof follows a technique similar to the one used in \citep{daniely2017depth} (see Remark \ref{remark:daniely} for a comparison) and essentially upper bounds the scalar product between the objective function and the network.

\begin{proposition}\label{prop:non_efficient_sphere}
Let $\bra*{f^{(d)}}_{d}$ a sequence of functions such that $f^{(d)} \in C(\S)$. Assume that for every $d$ there exists $I_d \subseteq \mathbb{N}$ such that
\begin{enumerate}
\item It holds that $\norm{f^{(d)}}_2 \leq O(d^M)\cdot \norm{P_{I_d} f^{(d)}}_2$ for some $M > 0$\,;
\item There exists a non-negative sequence $\bra{c_{d,k}}_{k\in I_d}$ such that $\norm{f_k^{(d)}}_\infty \leq c_{d,k}\sqrt{N_k^d} \norm{f^{(d)}}_2$ for all $ k \in I_d$ and such that $\parr*{\sum_{k\in I_d}c_{d,k}^2}^{1/2}\leq \epsilon^{d^\alpha} \cdot O(d^M)$ for some $\epsilon\in(0,1)$ and $\alpha>0$.
\end{enumerate}
Moreover, assume that $\norm{f^{(d)}}_\infty = O(1)$ and $\norm{f^{(d)}}_2 = \Omega( d^{-M}) $ .
Then the sequence $\bra*{f^{(d)}}_{d>2}$ is not universally approximable by one-hidden-neural networks.
\end{proposition}

The proof of Proposition \ref{prop:non_efficient_sphere} is reported in section \ref{sec:proof_non_efficient_sphere}. We discuss a few particular cases where the assumptions of Proposition \ref{prop:non_efficient_sphere} hold.
Let $\bra*{f^{(d)}}_{d>2}$ be a sequence of functions $f^{(d)} \in C(\S)$. 

\begin{example}[Constant control on $\ell_{\infty,2}$]

Assume that assumption 1 in Proposition \ref{prop:non_efficient_sphere} holds with $I_d = \bra*{k \in \mathbb{N}\st k \geq d^2}$ and that $\norm{f^{(d)}}_2 =\Omega( d^{-M} )$ for some constant $M>0$. If it holds that
$$
\ell_{\infty,2}(f_k^{(d)}) \leq \bar{\ell}
$$
for all $k \geq d^2$ for some constant $\bar{\ell} \geq 1$, then it is easy to check that Proposition \ref{prop:non_efficient_sphere} holds. This condition could be thought as the spherical harmonic components of the function $f^{(d)}$ being uniformly spread for high energy ($k \geq d^2$). Indeed assumption 2 holds with
$$
c_{d,k} \doteq \frac{\bar{\ell}}{\sqrt{N_k^d}}\frac{\norm{f_k^{(d)}}_2}{\norm{f^{(d)}}_2}
$$
since
$$
\sum_{k = d^2}^\infty c_{d,k}^2 \leq \frac{\bar{\ell}^2}{N_{d^2}^d} = O( d^{3-d} )~.
$$
This is similar to the condition used in \citep{daniely2017depth}, discussed in the remark below.

\end{example}

\begin{remark}\label{remark:daniely}
Daniely \citep{daniely2017depth} showed a depth-separation result using a result similar to Proposition \ref{prop:non_efficient_sphere}. The difference in this case is that the author considers functions defined on $\S\times\S$. Although, since $L^2(\S\times \S) = L^2(\S)\otimes L^2(\S)$, the space $L^2(\S\times \S)$ admits a decomposition in spherical harmonics 
$$
L^2(\S\times \S) = \sum_{j,k=0}^\infty H_j^d \otimes H_k^d~.
$$
In particular, Daniely considers functions of the type 
$$
f^{(d)}: (\bx,\by) \in \S\times \S \mapsto h^{(d)}(\bx^T\by)
$$
for some $h^{(d)} \in C([-1,1])$. Such functions belong to $\sum_{k=0}^\infty H_k^d\otimes H_k^d$ and satisfy 
$$
\ell_{\infty,2}(f^{(d)}_{k,k}) \leq \bar{\ell}\cdot\parr*{N_k^d}^{1/2} = \bar{\ell}\cdot\parr*{N_k^d}^{-1/2}\cdot \ell_{k,k}^*
$$
where $\ell^*_{k,k} = \max_{f \in H_k^d\otimes H_k^d}\ell_{\infty,2}\parr*{f}$. The equation above resembles condition 2 in Proposition \ref{prop:non_efficient_sphere}, since it implies that
$$
\norm*{f_{k,k}^{(d)}}_\infty \leq \frac{\bar{\ell}}{\sqrt{N_k^d}} \frac{\norm{f_{k,k}^{(d)}}_2}{\norm{f^{(d)}}_2} \cdot \ell_{k,k}^* \cdot \norm{f^{(d)}}_2
$$
and since 
$$
c_d \doteq \parq*{ \sum_{k \geq k_d} \parr*{\frac{\bar{\ell}}{\sqrt{N_k^d}} \frac{\norm{f_{k,k}^{(d)}}_2}{\norm{f^{(d)}}_2}}^2 }^{1/2} \leq \frac{\bar{\ell}}{\sqrt{N_{k_d}^d}}
$$
which, for $k_d \geq d^2$ implies that $c_d \lesssim d^3 \,2^{-d} $. The proof is then concluded by choosing $I_d = \bra*{(k,k)\st k \geq k_d}$, since (using the same notations as in the proof of Proposition \ref{prop:non_efficient_sphere}), it holds 
$$
\norm{f_N - f^{(d)}}_2^2 \geq \norm{\mathcal{P}_{I_d}f^{(d)}}_2^2 - 2\sum_{(j,j) \in I_d}\sum_{i=1}^N \parr*{\ell_{j,j}^*}^{-1} \abs*{u_i}\norm{f_{j,j}^{(d)}}_\infty \norm{f_{j,j}^{\sigma_i,\bw_i}}_2
$$
which is an equivalent of formula \eqref{eq:lb_sh}.
\end{remark}

\begin{example}

Assume that assumption 1 in Proposition \ref{prop:non_efficient_sphere} holds with $I_d = \bra*{k \in \mathbb{N}\st k \geq \rho d^\beta}$ for some $\rho >0$, $\beta > 0$ and that $\norm{f^{(d)}}_2 = \Omega ( d^{-M} )$ for some constant $M >0$. If it holds that
$$
\ell_{\infty,2}(f_k^{(d)}) \leq  \epsilon^k\cdot O( d^M )\cdot \sqrt{N_k^d} 
$$
for all $k \geq \rho d^\beta$ for some constant $M > 0$, then Proposition \ref{prop:non_efficient_sphere} holds, since
$$
\sum_{k = \rho d^\beta}^\infty \epsilon^k = \frac{\epsilon^{\rho d^\beta}}{1 - \epsilon} ~.
$$
This condition could also be thought as the spherical harmonic components of the function $f^{(d)}$ being uniformly spread for high energy ($k \geq d^2$), although in this case the spreadness is required to increase exponentially, as the degree increases, with respect to the maximum concetration achievable (that is $\parr{N_k^d}^{1/2}$).

\end{example}

\begin{example}[Invariant functions]
Finally, we show that certain symmetry assumptions can imply energy spreadness. Consider the case of a sign-invariant function $f \in C(\S)$, that is such that $f(\bepsilon \circ \bx) = f(\bx)$ for every $\bepsilon \in \{\pm 1\}^d$ and $\bx\in\S$. 

\begin{lemma}\label{lemma:rademacher}
Let $f\in C(\S)$ be a sign-invariant function. If
\begin{equation}\label{eq:centered_frequency}
\norm{f_k}_\infty = \sup_{\bepsilon \in\{\pm 1\}^d}\abs*{f_k(\bepsilon)}
\end{equation}
for some $ k \geq 16 d^2$
then it holds 
$$
\norm{f_k}_\infty \leq 2\cdot 2^{-d/2}\sqrt{N_k^d}\norm{f_k}_2 ~.
$$
\end{lemma}
\begin{proof}[Proof]
Notice that since $f$ is sign-invariant, so is $f_k$. Consider the function 
$$
P:\bx\in \S\mapsto 2^{-d}N_k^d\sum_{\bepsilon \in \bra{\pm 1}^d} P_k^d(\bepsilon^T \bx) ~.
$$
The function $P$ satisfies $\norm{P}_2 \leq 2\cdot 2^{-d/2} \sqrt{N_k^d}$ (see Lemma \ref{lemma:high_energy_sparse}). Let $\bepsilon \in \bra{\pm 1}^d$. Then it holds
\begin{align*}
\norm{f_k}_\infty & = \abs*{f_k(\bepsilon)} = \abs*{\prodscal{f_k, P}} \leq \norm{P}_2 \norm{f_k}_2 \leq 2\cdot 2^{-d/2} \sqrt{N_k^d} \norm{f_k}_2 ~.
\end{align*}
This concludes the proof.
\end{proof}

The statement of the above lemma therefore says that if $f$ is sign-invariant and achieves maximum energy in a specific frequency then it satisfies Assumption 2 from Proposition \ref{prop:non_efficient_sphere}. Under polynomial decay of $\norm{f_k}_2$, it should be possible to relax the condition \eqref{eq:centered_frequency} to ask for the frequency $\bw^{(k)} \in [0,\infty)^d$ such that $\norm{f_k}_\infty = \abs*{f_k(\bw^{(k)})}$ to satisfy
$$
\inf_{j \in [d]}\abs*{w^{(k)}_j} \geq \mathrm{poly}(d^{-1})~.
$$
\end{example}



\subsection{Efficient approximation under a sparsity condition of the spherical harmonics decomposition}\label{sec:sh:concentrated}

Works by Barron \citep{barron1993universal,klusowski2018approximation} essentially show that efficient approximation holds under a sparsity condition on the Fourier transform of the function to approximate; more specifically, for $f\in L^1(\R^d)$, the rate of (uniform) approximation is controlled by the quantity $\int_{\R^d} \norm{\bw}_1^2\abs{\hat{f}(\bw)}\,d\bw$. In this section we show that an equivalent control can be determined for approximation on the sphere, in terms of spherical harmonics decomposition. For technical reason, the result is estabilished for functions in $\hat{H}^d \doteq H_1^d \oplus \bigoplus_{k=1}^\infty H_{2k}^d$ (which correspond to the space of function in $L^2_S$ whose odd part is linear) and mainly for ReLu activation. We briefly discuss extensions to different activation functions in Remark \ref{remark:sh_ub_activations}. Consider the space of homogeneous one-hidden-layer neural networks with ReLU activations:
$$
\mathcal{F}_N^{\mathrm{ReLU}, 0} = \bra*{ f : \bx\in\S\mapsto \sum_{k=1}^N u_k \parr*{\bw_k^T\bx}_+ \st \bu \in\R^N, \bw_k \in\S} ~.
$$
Since 
$$
\parr*{\bw^T\bx}_+ = \frac{1}{2}\abs*{\bw^T\bx} + \frac{1}{2}\parr*{\bw^T \bx}~,
$$
every function in $\mathcal{F}_N^{\mathrm{ReLU}, 0}$ is the sum of a linear function with an even one. In other words, $\mathcal{F}_N^{\mathrm{ReLU}, 0} \subset \hat{H}^d$. Since any linear function belongs to $\mathcal{F}_2^{\mathrm{ReLU}, 0}$, it is equivalent to consider the problem of approximating even functions by homogeneous one-hidden-layer neural networks with activation $\mathrm{abs}(x) = \abs*{x}$, that is, elements of the space  
$$
\mathcal{F}_N^{\mathrm{abs}, 0} = \bra*{ f : \bx\in\S\mapsto \sum_{k=1}^N u_k \abs*{\bw_k^T\bx} \st \bu \in\R^N, \bw_k \in\S} ~.
$$
To study this, consider the corresponding functional space 
$$
\mathcal{H}^1 \doteq \bra*{ h_\pi \st \text{$\pi$ is a signed even Radon measure} }
$$ 
where $h_\pi$ is defined to be the function
$$
h_\pi:\bx \in \S \mapsto \int_\S \abs*{\bw^T \bx}\,d\pi(\bw)~.
$$
The space $\mathcal{H}^1$ is a Banach space endowed with the norm $\gamma_1(h) = \inf_{h\st h = h_\pi}\norm{\pi}_1$. As discussed in the introduction, the space $\mathcal{H}^1$ consists of functions which are efficiently approximable by one-hidden-layer networks. More formally, the following holds.
\begin{theorem}[\cite{bourgain1989approximation}]\label{theo:sampling_H1}
Let $f \in \mathcal{H}^1$. Then it holds that
$$
\inf_{f_N \in \mathcal{F}_N^{\mathrm{abs},0}} \norm{f - f_N}_\infty  \leq c \frac{\gamma_1(f)}{N^{1/3}}
$$
where $c > 0$ is a numerical constant. Moreover, $f_N$ satisfying the bound can be chosen to satisfy $\gamma_1(f_N) \leq \gamma_1(f)$.
\end{theorem}
The question of interest can now be transposed to: which functions $f\in C(\S)$ have a (polynomially) small norm $\gamma_1(f)$? One way to approach this problem is by the so-called Blaschke–Levy operator. Consider the transformation
$$
T \varphi = \int_{\S}\abs*{\bx^T\by}\varphi(\by)\,dS(\by)
$$
for functions $\varphi \in C(\S)$. $T$ can be described in terms of spherical harmonics \citep{rubin1998inversion}  as 
$$
T\varphi = \sum_{k\geq 0\; \mathrm{even}} \sigma_k \varphi_k \quad \text{where} \quad \sigma_k =        \frac{(-1)^{1+ k/2}}{2\pi} \frac{ \Gamma((k-1)/2) \Gamma(d/2) }{\Gamma((k+d+1)/2) } ~.
$$
In particular, it holds that
the functional $T$ is an automorphism of $C^\infty_{even}(\S)$ (the set of even function in $C^\infty(\S)$) \citep{rubin1998inversion} . Clearly, its inverse can be defined in terms of spherical harmonics by 
$$
T^{-1}:\varphi\in C^\infty_{even}(\S) \mapsto  \sum_{k\geq 0\; \mathrm{even}} \sigma_k^{-1} \varphi_k  ~.
$$
The following is immediate.
\begin{proposition}\label{prop:F1-space-Tinverse}
For any $\varphi \in C^\infty_{even}(\S)$ it holds that $\varphi \in\mathcal{H}^1$ and 
$$
\gamma_1(\varphi) = \norm*{T^{-1}\varphi}_1~.
$$
\end{proposition}
Using these results, we can proceed similarly to the work \citep{ongie2019function} and obtain the following.
\begin{proposition}\label{prop:F1-space}
Let $f\in C(\S)$ even. It holds that $f \in \mathcal{H}^1$ if and only if 
\begin{equation}\label{eq:norm_H1_finite}
\sup_{\varphi\in C^\infty_{even}(\S)\st \norm{\varphi}_\infty \leq 1} \prodscal{T^{-1}\varphi, f} < \infty~.
\end{equation}
In this case, 
$$
\gamma_1\parr{f} = \sup_{\varphi\in C^\infty_{even}(\S)\st \norm{\varphi}_\infty \leq 1} \prodscal{T^{-1}\varphi, f} ~.
$$
\end{proposition}

The proof of Proposition \ref{prop:F1-space} is reported in section \ref{sec:proof_F1-space}.
Functions that satisfy equation \eqref{eq:norm_H1_finite} include all even functions in $C^{d+2}(\S)$ if $d$ is even and all even functions in $C^{d+3}(\S)$ if $d$ is odd \citep{weil1976centrally}. This is inline with existing results that show approximability by neural networks for functions whose regularity is proportional to the dimension $d$ (e.g. \citep{maiorov2000near}).

Given $f\in C(\S)$ even, the condition of  Proposition \ref{prop:F1-space}  is implied by the (weak) convergence (as $N\to\infty$) of the series 
$$
S_Nf = \sum_{k=0}^N \sigma_{2k}^{-1}f_{2k}
$$
to a finite signed measure $\pi$. In this case $f = h_\pi$. In particular, a stronger condition is convergence in $L^1(S)$. This is implied if it holds that 
\begin{equation}\label{eq:L1_H1_condition}
\sum_{k\geq 0\text{ even}} \abs*{\sigma_k}^{-1} \norm{f_k}_1 < \infty ~.
\end{equation}
Notice that, instead, the series converges in $L^2_S$ if and only if 
\begin{equation}\label{eq:L2_H2_condition}
\sum_{k\geq 0\text{ even}} \sigma_{k}^2 \norm{f_{k}}^2_2 <\infty ~.
\end{equation}
This is equivalent to asking that $f \in \mathcal{H}^2$, the RKHS given by the kernel function
$$
k : (\bx,\by)\in\S\times\S \mapsto \int_\S \abs*{\bx^T \bw}\abs*{\bw^T \by} \,dS(\bw)~.
$$
Since in this case $\mathcal{H}^2$ can be described as
$$
\mathcal{H}_2 \doteq \bra*{ h_\pi \st \text{$\pi$ is a signed even Radon measure with an $L^2_S$ density}  }~, 
$$
it is clear that $\mathcal{H}^1 \subset \mathcal{H}^2$. We refer to \citep{bach2017breaking} for more details about these statements. On the other hand, notice that the condition \eqref{eq:L1_H1_condition} is potentially much stronger than simply asking for $f \in \mathcal{H}^1$.

\begin{example}[Highly concentrated function]
Some computations show that 
\begin{equation}\label{eq:sigma_k_decay}
\abs*{\sigma_k}^{-1} \leq \Theta\parr*{ d^{3/4} k^2 \sqrt{N_k^d} } ~.
\end{equation} 
Using these observations it is then straightforward to prove the following.
\begin{proposition}
Let $\bra*{f^{(d)}}_{d}$ a sequence of even functions in $C(\S)$. Assume that there exist some constant $M,N >0$ constant such that
$$
\sqrt{N_k^d}\norm{f_k^{(d)}}_1 \leq O\parr*{ k^{M}d^N } \cdot \norm{f^{(d)}_k}_2 \quad 
\text{and} \quad 
\sum_{k=0}^\infty k^{M+2}\norm{f_k^{(d)}}_2 = O(d^N)~.
$$
 Then the sequence $\bra*{f^{(d)}}_{d\geq 2}$ is universally approximable by the space $\mathcal{F}_N^{\mathrm{abs},0}$.
\end{proposition}
\begin{proof}[Proof]
By Proposition \ref{prop:F1-space-Tinverse} and equation \eqref{eq:sigma_k_decay} above we get that
\begin{align*}
\gamma_1(f^{(d)}) & \leq \sum_{k\geq 0 \text{ even}} \abs*{\sigma_k}^{-1}\norm{f_k^{(d)}}_1 \leq \Theta(d^{N + 3/4})\sum_{k\geq 0 \text{ even}} k^{2+M} \norm{f_k^{(d)}}_2 \leq O(d^{3/4+2N})~.
\end{align*}
The application of Theorem \ref{theo:sampling_H1} concludes the proof.
\end{proof}

The proposition above requires essentially two conditions to hold. First, that the energy of the functions decreases fast enough (yet polynomially in $k$ and $d$). The second condition is that the Fourier components of the function are concentrated enough, that is they are polynomially close to the bound \eqref{eq:ell_2_1_bound}. 
We remark that this condition is  infact pretty strong; it requires the function $f$ to be band-limited. According to \citep{dai2016reverse}, it holds that
$$
\norm{f_k^{(d)}}_2 \leq C(d) k^{\frac{d-2}{4}}\norm{f_k^{(d)}}_1~,
$$
for some function $C(d)$.
Then $f^{(d)}$ would satisfy
$$
\sqrt{N^d_k}\leq \mathrm{poly}(k,d) \frac{\norm{f^{(d)}_k}_2}{\norm{f^{(d)}_k}_1} \leq \mathrm{poly}(k,d) k^{\frac{d-2}{4}} 
$$
Since $\sqrt{N_k^d} \geq c(d) k^{\frac{d-2}{2}} $ for some $c(d)$, this implies that $k^{\frac{d-2}{4}}\mathrm{poly}(k^{-1}) \leq H(d)$ for some function $H(d)$. It follows that $k$ must satisfy $k \leq K(d)$ for some $K(d)$. Although, the rate of the function $K(d)$ does not follow from \citep{dai2016reverse}; we conjecture that $K(d)$ behaves as a power of $d$.
\end{example}

\begin{example}[High energy zonal harmonics]
The properties discussed in this section indicate that high-energy only does not yield not-universal-approximability. As an `extreme' case, consider the case of a zonal harmonic $f(\bx) \doteq P_k^d(\bw^T \bx)$, for  $\bx,\bw\in\S$ where $\bw$  is fixed. Notice that $\norm{f}_\infty = 1$. It holds that
$$
\gamma_1(f) = \frac{\norm{f_k}_1}{\abs{\sigma_k}} \leq O(k^2d^{3/4})\sqrt{N_k^d} \norm{f_k}_1 \leq O(k^2d^{3/4})  \norm*{ \sqrt{N_k^d} f_k}_2 = O(k^2d^{3/4})~,
$$
which implies universal approximability by Theorem \ref{theo:sampling_H1}. 
Similarly, polynomial combinations of zonal harmonics can be well approximated, as expected.
\end{example}

\begin{remark}[Ridge functions]
For a single neuron network $f(\bx) = |\bw^T\bx|$, it holds $\norm{f}_\infty = 1$ and $\norm{f}_2 = d^{-1/2}$. The spherical components of $f$ are given by
$$
f_k(\bx) = N_k^d\parq*{ \parr*{T \parq*{P_k^d(\bx^T\cdot)}}(\bw) } = (\sigma_k N_k^d)  P_k^d(\bw^T\bx)~.
$$
In particular, it holds 
$$
1 = \gamma_1(f) = \norm*{\sum_{k \geq 0 \text{ even}} \sigma_k^{-1}f_k }_1
= \norm*{\sum_{k \geq 0 \text{ even}} N_k^d P_k^d(\bw^T\cdot) }_1~.
$$
Therefore, understanding how tight (or strong) condition \eqref{eq:L1_H1_condition} is highly correlated with understanding convergence of the series $\sum_{k \geq 0 \text{ even}} N_k^d \norm*{P_k^d(\bw^T\cdot) }_1$, or equivalently, computing $\norm*{P_k^d}_{\mu_d,1}$.
\end{remark}

\begin{remark}\label{remark:sh_ub_activations}
While the result of this section mainly concern approximation by homogeneous one-hidden-layer networks with the ReLU (or absolute value) activation, they can easily be extended to any other activation satisfying Assumption \ref{ass:activation}, under the same assumptions. Moreover, notice that, thanks to Theorem \ref{theo:sampling_H1}, universal approximation by $\mathcal{F}_N^{\mathrm{ReLU},0}$ is equivalent to universal approximation by $\mathcal{H}_1 \oplus H_1^d$.
\end{remark}

\vskip 0.2in
\bibliography{refs}

\appendix

\section{Proofs of depth-separation results}

\subsection{Proof of Theorem \ref{theo:lb}}
\label{sec:proof:theo:lb}

The proof of the lower bound follows the same strategy as \citep{eldan2016power}.
For sake of simplicity in the following we remove the dimension $d$ from the following notations: $\bw_d = \bw$ and $\bv_d = \bv$. In the following we always assume $d\geq 3$. Let $S\subseteq [d]$ a subset and let $\mathbf{I}_{S}$ be the truncated identity matrix defined as
\begin{align*}
    \mathbf{I}_S:= \sum_{s\in S} \mathbf{e}_{s} \mathbf{e}_{s}^\top \,. 
\end{align*}
Moreover, define the function $H_S(\mathbf{x})$ as 
\begin{equation}\label{eq:def_H}
    H_S(\mathbf{x}) \doteq \prod_{i: i\in S}\mathbf{1}_{x_i>0}\prod_{j: j\in [d]\backslash S}\mathbf{1}_{x_j\leq 0}~.
\end{equation}
Lastly, for a subset $S\subseteq [d],$ let $\mathbf{v}_S:=\bv+\mathbf{I}_S\bw$ and define the function $\sigma_{r,S}(\mathbf{x}) :=\sigma_r(\mathbf{v}_S^T\mathbf{x}) $. Therefore, the expression of $f_{r_d,\bw,\bv}$  can be rewritten as: 
\begin{equation}\label{eq:def_tilde_g}
\begin{split}
    {f_{r_d,\bw,\bv}}(\mathbf{x}) 
    = \sum_{S\subseteq [d]}  g_S(\mathbf{x})=  \sum_{S\subseteq [d]} H_S(\mathbf{x})\sigma_{r_d,S}(\mathbf{x})
\end{split}
 \end{equation}
where $g_S(\mathbf{x}):=H_S(\mathbf{x})\sigma_{r_d,S}(\mathbf{x})$. 
Let the space of $N$-units one-hidden-layer networks be
\begin{equation*}
\mathcal{F}_N = \bra*{ f_N: \bx\in\R^r\mapsto  \sum_{k=1}^N \sigma_k(\mathbf{a}_k^T\mathbf{x}) \st \ba_k \in \R^d,\, \sigma_k \text{ are 1-Lipschitz activations} }.    
\end{equation*}
Assume that
\begin{itemize}
\item[(A1)] it holds that $\tau_d\cdot r_d \geq \beta d^k$ for some constant $k\geq 1$;
\item[(A2)] it holds that $\eta > \log_2 \parr*{\norm{\psi}_1\sqrt{K/2}}$
\end{itemize}
Then, for large enough $d$, it holds
\begin{align}
\label{mainresres}
\inf_{f \in \mathcal{F}_N} \norm{f_{r_d,\bw,\bv} - f}_{\varphi}^2 \geq 1 - N \left( 2^{1-2\eta } K \| \psi \|_1^2 \right)^d  O(d \cdot \tau_d \cdot r_d )  ~,
\end{align}
where we denote
$$
\norm{g}_\varphi^2 \doteq \int_{\R^d}\abs*{g(\bx)}^2\varphi^2(\bx)\,d\bx
$$
for $g \in L^2_{\varphi^2}$. 
In particular, if $N \simeq \mathrm{poly}(d)$, then the error \eqref{mainresres} tends to $1$ as $d\to \infty$. 

To show equation \eqref{mainresres}, we proceed as follows.
Let $\mathcal{F}=\{ \widehat{f \varphi} \st f \in \mathcal{F}_1 \}$, and denote by $F := \widehat{\varphi \cdot f_{r_d,\bw,\bv}} = \hat{f}_{r_d,\bw,\bv} * \hat{\varphi}$.
Since $\hat{\varphi}$ has compact support in $[-K,K]^d$ and the Fourier transform of a one-unit shallow network $f(\mathbf{x}) = \sigma( \mathbf{x}^T \mathbf{a})$ has support in the line $\{ \bxi \st \bxi = \alpha \mathbf{a},\, \alpha \in \mathbb{R} \}$, 
it follows that any function in $\mathcal{F}$ is supported in a tube $T=\{ \bxi\st \bxi = \alpha \mathbf{a}+[-K,K]^d,\, \alpha \in \mathbb{R} \}$ of radius $K$. 
For each tube $T$ of radius $K$, we consider $\mathcal{T}_T=\{\phi \in L^2 \st \text{ supp}({\phi}) \subseteq T \}$ and
$$\kappa \doteq \sup_{T\text{ tube of radius }K}  \norm{ P_{\mathcal{T}_T}(F) }_2  ~,$$
where $P_{\mathcal{T}_T}(F) = \mathrm{argmin}_{h\in \mathcal{T}_T} \|h-F\|^2_2. $ We claim that 
\begin{equation}
\label{eq:kappa_result}
  \inf_{f \in \mathcal{F}_N} \| f_{r_d,\bw,\bv} - f \|_\varphi^2 \geq 1 - N \kappa^2~.  
\end{equation}
Indeed, given $f \in \mathcal{F}_N$, denote by $T_1, \dots T_N$ the associated $N$ tubes, 
and by $\mathcal{T}_{T_1,\dots T_N} = \bigoplus_{k\in[N]} \mathcal{T}_{T_k}$ the corresponding subspace spanned by $\mathcal{T}_{T_k}$, $k\in[N]$. Then, by using the isometry of the Fourier transform, we have that
\begin{align}
\inf_{f \in \mathcal{F}_N} \| f - f_{r_d,\bw,\bv} \|_\varphi^2 &= \inf_{f \in \mathcal{F}_N} \| \widehat{f\varphi} - F \|^2_2\nonumber\\
&\geq \inf_{T_1,\dots T_N} \inf_{h \in \mathcal{T}_{T_1,\dots T_N} }  \| h - F \|^2_2\nonumber \\
&= \inf_{T_1,\dots T_N}   \| P_{\mathcal{T}_{T_1,\dots T_N}} (F) - F \|^2_2\nonumber\\
&= \inf_{T_1,\dots T_N} ( \| F \|^2_2 - \| P_{\mathcal{T}_{T_1,\dots T_N}} (F) \|^2_2)~.\label{eq:inf_bd_fg_tmp}
\end{align}
Now, observe  that 
$\sup_{T_1,\dots,T_N}\|P_{\mathcal{T}_{T_1,\dots T_N}} (F)\|^2_2 \leq N\sup_{T}\|P_{\mathcal{T}_{T}} (F)\|^2_2$.  
Equation \eqref{eq:inf_bd_fg_tmp}  therefore becomes
\begin{eqnarray*}
\inf_{f \in \mathcal{F}_N} \| f -  f_{r_d,\bw,\bv} \|_\varphi^2&\geq& \| F \|^2_2 - N \sup_T \| P_{\mathcal{T}_T}(F) \|^2_2 ~,
\end{eqnarray*}
which proves \eqref{eq:kappa_result} by plugging in the definition of $\kappa$ and recalling that $\|F \|^2_2 = \| f_{r_d,\bw,\bv} \|_\varphi^2 = 1$ by Parseval. 
To establish (\ref{mainresres}), it is therefore sufficient to prove that 
\begin{equation}
\label{kappares}
    \kappa^2 \leq  \left(\| \psi \|_1^2 2^{1-2\eta} K \right)^d  O(d \cdot \tau_d \cdot r_d )~. 
\end{equation}
The rest of the proof will be devoted to establishing a sufficiently sharp upper bound for $\| P_{\mathcal{T}_T}(F) \|_2$. 
 Observe that 
 $P_{\mathcal{T}_T}(F)$ is simply obtained by setting to zero all frequencies of $F$ outside $T$. 
We start by computing an upper bound on $\abs{F(\bxi)}$. 
We claim the following.
\begin{lemma}
\label{lemmaboundF} It holds that
\begin{eqnarray}
\label{vuu0}
    |F(\bxi)| &\leq& \frac{\| \varphi \|_1}{2^d}\sum_{S\subseteq [d]} \prod_{j=1}^d \min\left( 1, \frac{2K}{ \pi (|\xi_j - \xi_{S,j}|-K)_+}\right)~.
\end{eqnarray}
\end{lemma}
Let $D(\bxi) \doteq \sum_S D_S(\bxi)$, with 
$D_S(\bxi) \doteq \prod_{j=1}^d \min\left(1, \frac{2K}{\pi(|\xi_j - \xi_{S,j}| - K)_+ } \right)$,
so that from Lemma \ref{lemmaboundF} we have 
\begin{equation}
\label{coppa}
|F(\bxi)| \leq 2^{-d} \| \varphi \|_1 D(\bxi) ~.   
\end{equation}
Recall that $\tau_d = \sup_{S\in[d]} \| \mathbf{v}_S \|_\infty$. 
Given $\bxi$ non-zero, 
we claim the following.
\begin{lemma} It holds that
\label{lemmatwo}
\begin{equation}
\label{col}
D(\bxi) \leq C_{K,\gamma} 2^{d(1 - \eta)} \min\left\{1, 2K (\pi (\| \bxi\|_\infty - r_d\tau_d - K)_+)^{-1} \right\}~,
\end{equation}
where $C_{K,\gamma} = 2\exp\parr*{\sqrt{\frac{8K}{\pi\gamma}}}$.
\end{lemma}
Now, pick any arbitrary non-zero direction $\bnu$ such that $\| \bnu \|_\infty =1$. 
Let 
\begin{equation}
\label{eq:tubeinfty}
T = \{ \bxi\st  \inf_{\alpha \in \mathbb{R}} \| \bxi - \alpha \bnu \|_\infty \leq K \}    
\end{equation}
denote the tube of radius $K$ in the direction $\bnu$.
It holds that
\begin{eqnarray}
\label{eq:tubedecomp}
\int_T D(\bxi)^2 d\bxi &=& \underbrace{\int_{T \cap \left\{\| \bxi \|_\infty \leq 2 \tau_d r_d \right\}} D(\bxi)^2 d\bxi   }_{t_1} +  \underbrace{\int_{T \cap \left\{\| \bxi \|_\infty > 2 \tau_d r_d\right\} } D(\bxi)^2 d\bxi   }_{t_2} ~.
\end{eqnarray} 
In order to control the two terms $t_1$ and $t_2$, we use the following lemma to upper bound the measure of a $\ell_\infty$-cylinder.

\begin{lemma}
\label{lemmatubebound}
Let $T$ be an $\ell_\infty$-tube of radius $K$ as defined in \eqref{eq:tubeinfty}. If $\mu$ denotes the $d$-dimensional Lebesgue measure, 
then 
\begin{equation}
\label{eq:tubeb1}
    \mu\left( T \cap [-R,R]^d \right) \leq 8e^2  (d-1) (K+R) (2K)^{d-1}~.
\end{equation}
Moreover, if $g: \mathbb{R} \to \mathbb{R}$ is in $L^1(\mathbb{R})$ and  non-increasing, then  
\begin{equation}
\label{eq:tubeb2}
    \int_{T\cap \bra{\norm{\bxi}_\infty > R}} g( \| \bxi \|_\infty) \, d\bxi \leq 4 e^2 (d-1) (2K)^{d-1} \int_{R - K(2 + 3/(d-1))}^\infty g(u) du~,
\end{equation} 
as long as $R > K(2 + 3/(d-1))$. 
\end{lemma}
From \eqref{col} and \eqref{eq:tubeb1}, the first term of \eqref{eq:tubedecomp} can be bounded as 
\begin{align}
\label{mimi1}
t_1 & \leq 8 e^2 C_{K,\gamma}^2 2^{2d(1 - \eta) + (d-1)} K^{d-1} (d-1) (K + 2  \tau_d r_d) \nonumber \\
& \leq D_{K,\gamma}^{(1)} \cdot d \cdot \parr*{  \tau_d r_d} \parr*{ 2^{2(1-\eta)+1}K  }^d
\end{align}
for $D^{(1)}_{K,\gamma} = 16e^2 K^{-1} C_{K,\gamma}^2 $ and $d$ large enough, such that $2 \tau_d r_d \geq K$.
Similarly, using \eqref{eq:tubeb2}, the second term $t_2$ in turn can be bounded as
\begin{align}
\label{mimi2}
t_2 &\leq 8e^2\pi^{-2} C_{K,\gamma}^2 d \parr*{ 2^{2(1-\eta)+1}K  }^d \int_{2\tau_d r_d - K(2 + 3/(d-1))} \parr*{u - \tau_dr_d - K}^{-2}\,du  \nonumber \\
& = 8e^2\pi^{-2} K C_{K,\gamma}^2 d \parr*{ 2^{2(1-\eta)+1}K  }^d \parr*{ \tau_d r_d - 3K(1 + 1/(d-1))}^{-1}  \nonumber
\\
& \leq D^{(2)}_{K,\gamma} \cdot  d \cdot \parr*{ 2^{2(1-\eta)+1}K  }^d~,
\end{align}
for $D^{(2)}_{K,\gamma} = 16e^2\pi^{-2}C_{K,\gamma}^2$ and and $d$ large enough, such that $\tau_d r_d \geq 10K$.
Thus, collecting \eqref{mimi1} and \eqref{mimi2} and using \eqref{coppa}, we obtain 
\begin{align*}
\int_T \abs{F(\bxi)}^2 d\bxi & \leq \| \varphi\|^2_1 \cdot 2^{-2d}  (t_1 + t_2) \\
& \leq d \cdot \norm{\varphi}_1^2  \parr*{2^{1 - 2\eta} K }^d\parr*{D_{K,\gamma}^{(1)}\tau_d r_d + D_{K,\gamma}^{(2)} }\\
& \leq  D_{K,\gamma} \cdot d \cdot \norm{\varphi}_1^2  \parr*{2^{1 - 2\eta} K }^d \max(1, \tau_dr_d)~,
\end{align*}
where 
$$
D_{K,\gamma} \doteq D^{(1)}_{K,\gamma} + D^{(2)}_{K,\gamma} = 32 \exp\parr*{2 + \sqrt{\frac{8K}{\pi\gamma}}}\parr*{\pi^{-2} + K^{-1}}~.
$$
It follows that
$$
\| P_{\mathcal{T}_T}(F) \|^2_2 =  \int_T |F(\bxi)|^2 d\bxi   
     \leq D_{K,\gamma} \cdot (d \cdot \tau_d\cdot r_d) \cdot   \left( \| \psi \|_1^2\, 2^{1 -2\eta} K \right)^d ~,
$$
as long as $d \geq \parq*{\beta^{-1} \max(1, 10K)}^{1/k}$ (where $\beta$ and $k$ satisfy $\tau_dr_d \geq \beta d^k$). 
We have just established \eqref{kappares}, and this concludes the proof of the theorem. In the remaining part of this section we prove the auxiliary lemmas used above.

\begin{proof}[Proof of Lemma \ref{lemmaboundF}]
We start by computing $\hat{f}_{r_d,\bw,\bv}$. 
From the definition of $\sigma_r$, it follows that 
\begin{equation}
\label{eq:basicsigma_d}
\hat{\sigma}_{r,S}(\mathbf{\bxi}) = \delta( \bxi - r \mathbf{v}_S)~,
\end{equation}
which combined with the definition of $H$ yields 
\begin{equation*}
\hat{f}_{r_d,\bw,\bv}(\bxi) = \sum_{S\subseteq[d]} \parr*{\hat{H}_S * \hat{\sigma}_{r_d,S}}(\bxi) = \sum_{S\subseteq[d]} \hat{H}_S(\bxi - r_d \mathbf{v}_S)~.  
\end{equation*}
Let $\bxi_S \doteq r_d \mathbf{v}_S$. It holds that
\begin{align}
\label{eq:boundwhole0}
F(\bxi) &= \int_{\R^d} \hat{f}_{r_d,\bw,\bv}(\bnu) \hat{\varphi}(\bxi - \bnu) \,d\bnu 
= \sum_{S\subseteq[d]} \int_{\R^d} \hat{H}_S(\bnu - \bxi_S) \hat{\varphi}(\bxi - \bnu) \,d\bnu \nonumber \\
&= \sum_{S\subseteq[d]} \underbrace{\int_{\R^d} \hat{H}_S(\bnu ) \hat{\varphi}(\bxi - \bxi_S - \bnu) \,d\bnu}_{\doteq F_S(\bxi - \bxi_S)}~.
\end{align}
We can now bound each term $F_S$ separately. It holds that
\begin{equation}
\label{eq:boundwhole}
F_S(\bxi) = \int \hat{H}_S(\bnu ) \hat{\varphi}(\bxi - \bnu) d\bnu 
= \int H_S(\mathbf{\bx}) e^{2 i\pi \bxi^T \mathbf{x}} \varphi(\mathbf{x}) d\mathbf{x}~ 
=\prod_{j=1}^d F_j(\xi_j)
\end{equation}
where 
\begin{equation}
\label{eq:bubu}
F_j(t) = \int_\R \mathbbm{1}\bra*{ \epsilon_j x > 0} e^{2 i \pi t x} \psi(x) \,dx\,,~    
\end{equation}
with $\epsilon_j = \pm 1$. Assume without loss of generality that $\epsilon_j = 1$.
Observe that $F_j = \check{Q}$, where 
\begin{equation}
Q(u) = \mathbbm{1}\bra*{u > 0} \psi(u)~.    
\end{equation}
Since $\psi \in L^1(\mathbb{R})$ and its Fourier 
transform $\hat{\psi}$ has compact support in $[-K,K]$, it holds that
\begin{equation}
\label{eq:fourierboundswindow}
| \hat{\psi} (\tau) | \leq \| \psi \|_1~\text{ for }~\tau \in [-K,K] \quad\text{and}\quad  \hat{\psi} (\tau) = 0 ~\text{ for }~|\tau|>K~. 
\end{equation}
On the one hand, since $\psi$ is even, it holds, by directly bounding \eqref{eq:bubu}, that 
\begin{equation}
    | F_j(t)| \leq \frac{1}{2} \int_\R |\psi(u)| du = \frac{1}{2}\| \psi\|_1 \quad\text{for all }t~,
\end{equation}
and from \eqref{eq:fourierboundswindow} and the Hilbert transform of $Q$ we deduce on the other hand that 
\begin{equation*}
    | F_j(t) |  = \frac{1}{2\pi} \left| \int_{-K}^K \frac{\hat{\psi}(\tau)}{t-\tau} d\tau \right|\leq \frac{2K \| \psi\|_1}{(2\pi)(|t|-K)}\quad  \text{ for } |t| > K~,
\end{equation*}
so that it follows that
\begin{equation}
\label{eq:boundcoordinate}
    |F_j(t) | \leq \frac{\| \psi\|_1}{2}\min\left( 1, \frac{2K}{\pi (|t|-K)_+}\right)~.
\end{equation}
Thus, from equations \eqref{eq:boundwhole0}, \eqref{eq:boundwhole} and \eqref{eq:boundcoordinate} it follows that
\begin{eqnarray*}
|F(\bxi)| &\leq& \sum_{S\subseteq[d]} |F_S(\bxi - \bxi_S)|  \\
    &\leq& \frac{\| \varphi \|_1}{2^d}\sum_{S\subseteq[d]} \prod_{j=1}^d \min\left( 1, \frac{2K}{ \pi (|\xi_j - \xi_{S,j}|-K)_+}\right)~, 
\end{eqnarray*}
which proves Lemma \ref{lemmaboundF}. 
\end{proof}

\begin{proof}[Proof of Lemma \ref{lemmatwo}]
Let define for any $\bxi \in \mathbb{R}^d$ and $\lambda >0$ 
\begin{equation*}
\mathsf{n}( \bxi, \lambda) \doteq  \abs*{ \{j\in[d]\st |\xi_j| > \lambda \} }~. 
\end{equation*}
Recall that $\mathbf{v}_S = \bv + \bI_S\bw$
and $\bxi_S = r_d \mathbf{v}_S$.
Observe that $\bxi_{S} - \bxi_{S'} = r_d (\mathbf{I}_S - \mathbf{I}_{S'})\mathbf{w}$, so 
\begin{equation}\label{eq:xixiSS_pr}
|\xi_{S,j} - \xi_{S',j}| = \left \{ 
\begin{array}{cc}
r_d | w_j| & \text{ if } j \in (S \cup S')\setminus (S \cap S') \\
0 & \text{ otherwise }
\end{array} \right. ~.
\end{equation}
If $\mathsf{d}(S,S')$ denotes the Hamming distance between two subsets $S,S'$, then for all $S,S'$, the following holds.
\begin{lemma}\label{lem:hamm_n} It holds that
\begin{equation}
\label{kiki}
\mathsf{n}(\bxi_S - \bxi_{S'}, \gamma d^2) = \mathsf{d}(S \cap \Omega_d, S' \cap \Omega_d)~.
\end{equation}
\end{lemma}
This immediately implies that
\begin{equation}
\label{cucu}
\mathsf{n}\left( \bxi - \bxi_S, \frac{\gamma d^2}{2}\right) + \mathsf{n}\left( \bxi - \bxi_{S'}, \frac{\gamma d^2}{2}\right) \geq \mathsf{d}(S \cap \Omega, S' \cap \Omega) \quad \text{ for all } \bxi ~\text{ and }~ S \neq S'~. 
\end{equation}
Indeed, if that was not the case, applying the triangle inequality coordinate-wise would contradict equation \eqref{kiki}. 
The first upper bound is obtained by first noticing that, for $d > 2\sqrt{K / \gamma}$, it holds 
\begin{equation*}
D_S(\bxi) \leq \left( \pi ( {\gamma} d^2/2 - K) / (2K) \right)^{-\mathsf{n}(\bxi - \bxi_S, {\gamma} d^2/2) } \quad \text{ for all } S ~\text{ and }~ \xi~.    
\end{equation*}
Now, defining
$S^*_\bxi = \arg\min_{S\subseteq[d]} \mathsf{n}(\bxi - \bxi_S, {\gamma} d^2/2)$,
from \eqref{cucu} it follows  that 
\begin{equation*}
\mathsf{n}(\bxi - \bxi_S,  {\gamma}d^2/2) \geq \frac{\mathsf{d}(S \cap \Omega_d, S' \cap \Omega_d)}{2} \quad\text{ for all } S \neq S^*_\bxi 
\end{equation*}
and thus, for $d > 2\sqrt{K / \gamma}$, it holds
\begin{align}
\label{huihui}
D(\bxi) &= D_{S^*_\bxi}(\bxi) + \sum_{S \neq S^*_\bxi} D_S(\bxi) \nonumber \\
&\leq  D_{S^*_\bxi}(\bxi) + \sum_{s=1}^{|\Omega_d|} \sum_{S\st \mathsf{d}(S \cap \Omega_d,S^*_\bxi \cap \Omega_d)=s}
\left( \pi (  {\gamma}d^2/2 - K) / (2K) \right)^{-s/2 } \nonumber \\
&\leq  D_{S^*_\bxi}(\bxi) + 2^{d - |\Omega_d|}\sum_{s=1}^{|\Omega_d|} \binom{|\Omega_d|}{s}
\left( \pi (  {\gamma}d^2/2 - K) / (2K) \right)^{-s/2 } \nonumber \\
& \leq  1 + 2^{d - |\Omega_d|} \left(1 + \frac{1}{\sqrt{\pi (  {\gamma}d^2/2 - K)/ (2K)}} \right)^{|\Omega_d|}~ \nonumber \\
& \leq C_{K,\gamma} 2^{d(1 - \eta)} 
\end{align}
since $|\{S\st \mathsf{d}(S \cap \Omega_d,S^*_\bxi \cap \Omega_d)=s\}| \leq 2^{d - |\Omega_d|} \binom{|\Omega_d|}{s}$. The term $C_{K,\gamma}$ is a constant that depends only on $K$ and $\gamma$; in particular, we can choose $C_{K,\gamma} = 2 \exp\parr*{\sqrt{\frac{8K}{\pi\gamma}}}$.
The second upper bound is obtained using the above argument as follows. 
Let $q_\bxi =\arg\max_j |\xi_j|$. Since $\norm{\bxi_S}_\infty \leq r_d\tau_d$ for any $S\subseteq [d]$, it holds that
\begin{align}
\label{hehe}
    D(\bxi ) &\leq \sum_{S\subseteq[d]} \frac{2K}{\pi(| \xi_{q_\xi} - \xi_{S,q_\xi}| - K)_+} \cdot \prod_{j \neq q_\xi} \min\left(1, \frac{2K}{\pi(|\xi_j - \xi_{S,j}| - K)_+ } \right) \nonumber \\
    &\leq 2K (\pi (\| \bxi\|_\infty -\tau_d r_d - K)_+ )^{-1} \sum_{S\subseteq[d]} \prod_{j \neq q_\xi} \min\left(1, \frac{2K}{\pi(|\xi_j - \xi_{S,j}| - K)_+ } \right) \nonumber \\
    &\leq C_{K,\gamma}  2K (\pi (\| \bxi\|_\infty - \tau_d r_d - K)_+)^{-1} \cdot 2^{d(1 - \eta)}
\end{align}
by noticing that the argument leading to \eqref{huihui} can now be 
repeated for the $(d-1)$-dimensional vector $\check{\bxi}=(\xi_1,\dots, \xi_{q_\bxi-1}, \xi_{q_\bxi+1},\dots \xi_d)$, so that 
\begin{equation}
\mathsf{n}(\check{\bxi} - \check{\bxi}_S,  {\gamma}d^2/2) \geq \frac{\mathsf{d}((S \cap \Omega_d)\setminus \{ q_\bxi \}, (S' \cap \Omega_d ) \setminus  \{ q_\bxi \} )}{2} \quad\text{ for all } ~ S \neq S^*_\bxi 
\end{equation}
which proves \eqref{hehe} and concludes the proof of Lemma \ref{lemmatwo}.
\end{proof}

\begin{proof}[Proof of Lemma \ref{lem:hamm_n}]
In fact, it holds that the two sets $A_1:=\{j\in[d]\st |\xi_{S,j}-\xi_{S',j}|\geq \gamma d^2\}$ and $A_2:=  \{j\in[d]\st j \in (S\cap \Omega_d) \backslash (S'\cap \Omega_d)\}$ are equal. Let $j\in A_1$. Then $|\xi_{S,j}-\xi_{S',j}|>\gamma d^2$. Since this quantity is nonzero, equation \eqref{eq:xixiSS_pr} indicates that therefore $j\in S\backslash S'$ without loss of generality. Moreover, $|\xi_{S,j}-\xi_{S',j}| = r_d |w_j|$ which implies that $r_d |w_j|>\gamma d^2$ and $j\in\Omega_d$. We conclude that $j\in (S\cap \Omega_d)\backslash (S'\cap\Omega_d) $ which implies that $j\in A_2$. Now, let $j\in A_2$. Then,  without loss of generality, $j\in (S\cap \Omega_d) \backslash (S'\cap \Omega_d)$. Then, it holds $r |w_j|>\gamma d^2 $ since $j\in S\backslash S'$ according to \eqref{eq:xixiSS_pr} and $|\xi_{S,j}-\xi_{S',j}| =r_d |w_j|.$ Combining these two facts, it follows that $|\xi_{S,j}-\xi_{S',j}|>\gamma d^2$ which means that $j\in A_2.$ 
\end{proof}

\begin{proof}[Proof of Lemma \ref{lemmatubebound}]
Let 
\begin{align*}
T_R(\bnu) & = T(\bnu) \cap [-R,R]^d \\ & = \{ \bxi \st  \inf_{\alpha \in \mathbb{R}} \sup_{j\in [d]} |\xi_j- \alpha \nu_j |\leq K \text{ and } \| \bxi \|_\infty \leq R \}~.    
\end{align*} 
The aim is to upper bound the volume of $T_R(\bnu)$ for any $\bnu.$ 
Assume, without loss of generality, that $\| \bnu \|_\infty =1$. 
The cut-off tube $T_R(\bnu)$ can be covered with $\ell_{\infty}$-balls of radius $K'= \vartheta K$ centered along the ray defined by $\bnu$, that is 
\begin{equation}
\label{eq:tubecover}
T_R(\bnu)\subseteq \bigcup_{j=-\floor{(K + R)/s}}^{\floor{(K + R)/s}} \parr*{js\bnu + [-\vartheta K,\vartheta K]^d} ~.  
\end{equation}
Now, we optimize both the sampling rate $s \in (0,K)$ and the radius ratio $\vartheta \geq 1$ while satisfying \eqref{eq:tubecover}. 
Given $s$, let us first compute the smallest admissible $\vartheta$.
Any $\mathbf{x}\in T_R(\bnu)$ satisfies
$$\|\mathbf{x}-(j+y)s\bnu\|_{\infty}\leq K$$
for some $j \in \mathbb{N}$ and $\abs{y} < 1$. 
This implies that $\norm{\bx - js\bnu}_\infty \leq K + ys \leq K + s$. Therefore an admissible $\vartheta$ is given by the solution of $K + s = \vartheta K$, that is $\vartheta=1+sK^{-1}$.
Now, 
the volume of 
$$
S_R \doteq  \bigcup_{j=-\floor{(K + R)/s}}^{\floor{(K + R)/s}} \parr*{
js\bnu + 
\parq*{ -\left(1+\frac{s}{K}\right)K,\left(1+\frac{s}{K}\right)K}^d
} 
$$
is upper bounded by 
$$
l(s)\doteq 4 \frac{K+R}{s}\left(2(K+s)\right)^d~. 
$$ 
Minimizing over $s$ gives $s = \frac{K}{d-1}$.
Therefore, for all $\bnu\in \mathbb{R}^d$, it holds $$
T_R(\bnu)\leq (K+R)K^{d-1}(d-1)\left(1+\frac{1}{d-1}\right)^d \leq (K+R)(d-1)K^{d-1} e^2~,
$$ 
which proves \eqref{eq:tubeb1}. 
Equation \eqref{eq:tubeb2} is established analogously. Let $T_{>R}(\bnu) = T(\nu) \cap \bra*{\bxi \st \norm{\bxi}_\infty > R}$. Then we have that 
$$
T_{>R}(\bnu) \subseteq \bigcup_{j \geq \floor{\frac{R - K}{s}}} \parr*{ js\bnu + [-(K+s),(K+s)]^d }~,
$$
where we set $s = K/(d-1)$. Since $g$ is non-increasing, it follows that
\begin{align*}
\int_{T_{>R}(\bnu)} g(\norm{\bxi}_\infty) \,d\bxi & \leq \sum_{\abs{j} \geq \floor{\frac{R-K}{s}}} \int_{\norm{\bxi - js\bnu}_\infty \leq K+s} g(\norm{\bxi}_\infty)\,d\bxi \\
& \leq 2 (2(K+s))^d \sum_{j \geq \floor{\frac{R-K}{s}}}  g( js - (K+s) ) \\
& \leq 2 (2(K+s))^d \sum_{j \geq \floor{\frac{R-K}{s}}}  \frac{1}{s}\int_{(j-1)s - (K+s)}^{js - (K+s)} g(u)\,du \\
& \leq \frac{2(2(K+s))^d }{s} \int_{R-K - 2s - (K+s)}^\infty g(u)\,du \\
& \leq \frac{2e^2(d-1) (2K)^d }{K} \int_{R - K(2 + 3/(d-1))}^\infty g(u)\,du~.
\end{align*}
This establishes \eqref{eq:tubeb2} and concludes the proof.
\end{proof}

\subsection{Proof of Theorem \ref{theo:depth-sep}}\label{sec:proof:theo:depth-sep}

The proof consists in approximating the activation $\sigma_r$ using Assumption 1.2 on $\sigma$. Since $\sigma_r$ is $(2\pi r)$-Lipschitz, we obtain that there exists, for any $r,Q>0$, $\alpha_k,\beta_k\in\R$ such that over the interval $[-Q, Q]$ it holds
$$
\sup_{|t| \leq Q} \left| \sigma_r(t) - \sum_{k=1}^N \alpha_k \sigma(t - \beta_k) \right| \leq \frac{2 Q r }{N}
$$
as well as 
$$\left| \sum_{k=1}^N \alpha_k \sigma(t - \beta_k) \right| \leq 1 + 2Qr/N \quad\text{for}~t \in \mathbb{R}~.
$$
Let $f_N \in \mathcal{F}_N^\sigma$ be defined as 
$$
f_N(\bx) = \sum_{k=1}^N \alpha_k \sigma\parr*{ r_d \parr*{\bv_d^T\bx + \bw_d^T\bx_+} - \beta_k}
$$
Now, let $\gamma_d = \norm{\bv_d}_1+ \|\mathbf{w_d} \|_1$ and 
$\tilde{Q_d} = \frac{Q_d}{\gamma_d}$, 
so that by definition when $\| \mathbf{x}\|_\infty \leq \tilde{Q_d}$ it holds that
$$| \bv^T_d\bx + \bw^T_d\bx_+ |  \leq Q_d~.$$
The approximation error can be decomposed as follows:
\begin{multline*}
    \int_{\R^d} (f_{r_d,\bw_d,\bv_d}(\mathbf{x}) - f_N(\mathbf{x}))^2 \varphi(\mathbf{x})^2 \,d\mathbf{x} = \\ =  \int_{\| \mathbf {x} \|_\infty \leq \tilde{Q}_d} (f_{r_d,\bw_d,\bv_d}(\mathbf{x}) - f_N(\mathbf{x}))^2 \varphi(\mathbf{x})^2 \,d\mathbf{x} + \int_{\| \mathbf{x} \|_\infty > \tilde{Q}_d} (f_{r_d,\bw_d,\bv_d}(\mathbf{x}) - f_N(\mathbf{x}))^2 \varphi(\mathbf{x})^2 \, d\mathbf{x} \\
    \leq  \frac{4 Q_d^2 r_d^2}{N^2} \| \varphi \cdot \mathbbm{1}_{B^d_{\tilde{Q}_d,\infty}} \|^2_2 + 4\parr*{1 + \frac{Q_dr_d}{N}}^2 ( \| \varphi\|^2_2 - \| \varphi \cdot \mathbbm{1}_{B^d_{\tilde{Q}_d,\infty}} \|^2_2 ) \\
     \leq  \frac{4 \tilde{Q}_d^2 \gamma_d^2 r_d^2}{N^2} \| \varphi \|^2_2 + 4\parr*{1 + \frac{Q_dr_d}{N}}^2  \left(1 - ( 1 - \alpha \,\tilde{Q}_d^{-1})^d\right) \\
    \leq \| \varphi \|^2_2 \left( \frac{4 \tilde{Q}_d^2 \gamma^2_d  r^2_d}{N^2} + 16\alpha  d \tilde{Q}^{-1}_d  \right) ~,
\end{multline*}
since $|\psi(x)|^2 \leq \alpha |x|^{-2} /2$ for some $\alpha > 0$, as long as $\tilde{Q}_d > \alpha$ and $N > r_d \tilde{Q_d}$. 
Optimizing this upper bound with respect to $\tilde{Q_d}$ gives 
$$
\tilde{Q}_d = \left(2\alpha d\frac{N^2}{r^2_d \gamma^2_d} \right)^{1/3},
$$
which results in 
\begin{equation*}
    \| f_{r,\bw,\bv} - f \|_\varphi^2  \lesssim \parr*{\frac{d\gamma_dr_d}{N}}^{2/3} ~,
\end{equation*}
as long as $N > \alpha r_d\gamma_d$. This concludes the proof.

\section{Proofs of \texorpdfstring{$\mathrm{poly}(d)$}{poly(d)} upper bounds}

\subsection{Proof of Lemma \ref{lemma:change_activation}}

We show this for the case $L(f^{(d)}) = 2$, but the proof it is analogous for the other cases. The function $f^{(d)}$ has the form
$$
f^{(d)}(\bx) = \bgamma^T_d \brho_2 \parr*{ \bW_d \brho_1 \parr*{ \bU_d \bx } }
$$
where $\brho_1^{(d)},\brho_2^{(d)}$ are component-wise activations satisfying Assumption \ref{ass:activation}, and $\bgamma_d \in \R^{q_d}$, $\bW \in \R^{q_d\times p_d}$, $\bU \in \R^{p_d\times d}$, with
$$
p_d, q_d, \norm{\bgamma}_\infty, \norm{\bW}_{F,\infty}, \norm{\bU}_{F,\infty} \leq \mathrm{poly}(d)~.
$$
Thanks to Assumption \ref{ass:activation}.2, there exists $\bA\in\R^{N p_d \times d}$, $\bB\in\R^{p_d\times Np_d}$, $\bc\in \mathbb{R}^{Np_d}$ such that 
$$
\sup_{\bx \in K} \abs*{\bgamma^T \brho_2 \parr*{ \bW \brho_1 \parr*{ \bU \bx } } - \bgamma^T \brho_2 \parr*{ \bW \bB \, \bsigma \parr*{ \bA \bx + \bc } }} \leq \frac{\epsilon}{2} 
$$
and 
$$
N, \norm{\bc}_\infty, \norm{\bB}_{F,\infty}, \norm{\bA}_{F,\infty} \leq \epsilon^{-1} \cdot \mathrm{poly}(d)~.
$$
Let $K_1 = \bra*{ \bB\bsigma ( \bA \bx + \bc ) \st \bx \in K }$; it holds $\mathrm{diam}(K_1) \leq \mathrm{poly}(d)$. Similarly as before, we get that there exists $\bD\in\R^{M q_d \times p_d}$, $\bE\in\R^{q_d\times Mq_d}$, $\bfun\in \mathbb{R}^{Mq_d}$ such that 
$$
\sup_{\by \in K_1} \abs*{\bgamma^T \brho_2 \parr*{ \bW \by } - \bgamma^T \bE \,\bsigma \parr*{ \bD \by + \bfun}} \leq \frac{\epsilon}{2} 
$$
and 
$$
M, \norm{\bfun}_\infty, \norm{\bE}_{F,\infty}, \norm{\bD}_{F,\infty} \leq \epsilon^{-1} \cdot \mathrm{poly}(d)~.
$$
By calling $\tilde{\bgamma} = \bE^T\bgamma$, $\tilde{\bW} = \bD\bW\bB$ and $\tilde{\bU} = \bU\bA$, we get that
$$
g^\sigma(\bx) \doteq \tilde{\bgamma}^T\bsigma(\tilde{\bW} \bsigma(\tilde{\bU}\bx + \bc) +\bfun )
$$
satisfies the statement of the theorem.

\subsection{Preliminary lemmas}

The first lemma is a known results in approximation theory.

\begin{lemma}{(Jackson's Theorem, Theorem 1.4 in \citep{rivlin1981introduction})}\label{lemma:jackson}
Let $f:[a,b]\to\R$ with modulus of continuity $\omega$. Then there exists a polynomial $p_n(t) = \sum_{k=0}^n p_kt^k$, $p_k\in\R$, such that
$$
\sup_{t\in[-r,r]}\abs*{f(t) -p_n(t)} \leq 6\,\omega\parr*{\frac{b-a}{2n}}~.
$$
\end{lemma}

The next lemma yields a worst approximation rate but allows us to control the coefficients of the polynomial. It is a small modification of Lemma 4 in \citep{safran2019depth}.
\begin{lemma}\label{lemma:safran}
Let $f:[-r,r] \to \R$ $(1,\alpha)$-Holder. Then for any $\epsilon > 0$ there exists a polynomial $p_n(t) = \sum_{k=0}^n r_kt^k$, $r_k\in \R$, of degree $n = \ceil*{\frac{4^{\frac{1}{\alpha}}r^\alpha}{\epsilon^{1 + \frac{2}{\alpha}}}}$ such that
$$
\sup_{t \in [-r,r]} \abs*{f(t) - p_n(t)} \leq \epsilon~.
$$
Moreover, $p_n$ can be chosen such that $\abs*{r_k} \leq 2^nr^{\alpha -k}$, $k \in [n]$, and $\abs*{r_0} \leq r^\alpha + \abs*{f(0)}$.
\end{lemma}
\begin{proof}[Proof]
Notice that we can assume $f(0) = 0$ without loss of generality. Define $g(t) = f(r(2t-1))$ for $t\in[0,1]$ and notice that $g$ is $((2r)^\alpha,\alpha)$-Holder. Also, define the $n$ Bernstein polynomial $b_{n,i}$, $i \in [0,n]$, as
$$
b_{n,i}(t) = \binom{n}{i}t^i(1-t)^{n-i}
$$
for $t \in [0,1]$.
Notice that they form a partition of unity. We define 
$$
g_n(t) = \sum_{i=0}^n g\parr*{\frac{i}{n}}b_{n,i}(t)~.
$$
We have that
\begin{align*}
\abs*{g_n(t) - g(t)} & \leq \sum_{i=0}^n b_{n,i}(t) \abs*{ g(t) - g\parr*{\frac{i}{n}}} \\ 
& = \sum_{i \st \abs*{\frac{i}{n} - t} < \epsilon } b_{n,i}(t) \abs*{ g(t) - g\parr*{\frac{i}{n}}} + \sum_{i \st \abs*{\frac{i}{n} - t} \geq \epsilon } b_{n,i}(t) \abs*{ g(t) - g\parr*{\frac{i}{n}}} \\
& \leq \epsilon^{\alpha} + 2r^{\alpha} \sum_{i \st \abs*{\frac{i}{n} - t} \geq \epsilon } b_{n,i}(t) \leq \epsilon^{\alpha} +  \frac{r^{\alpha}}{2n\epsilon^2}~.
\end{align*}
In particular $\frac{r^{\alpha}}{2n\epsilon^2} \leq \epsilon^\alpha$ if
$$
n \geq \frac{r^\alpha}{2\epsilon^{2+\alpha}}~.
$$
If we define $p_n(t) = g_n\parr*{\frac{t}{2r}+ \frac{1}{2}}$, then we have that
$$
\sup_{x\in[-r,r]}\abs*{f(t) - p_n(t)} \leq \epsilon
$$
if
$$
n \geq \frac{4^{\frac{1}{\alpha}}r^\alpha}{\epsilon^{1 + \frac{2}{\alpha}}}~.
$$
Finally, we want to upper bound the coefficients of $p_n$. Notice that we have
\begin{align*}
p_{n}(t) & = (2r)^{-n}\sum_{i=0}^n \binom{n}{i}g\parr*{\frac{i}{n}}(t+r)^i(t-r)^{n-i} ~.
\end{align*}
It follows that the coefficients of $p_n$ can be bounded by those of 
\begin{align*}
(2r)^{-n}\sum_{i=0}^n \binom{n}{i}\abs*{g\parr*{\frac{i}{n}}} (t+r)^n \leq r^{\alpha -n}(t+r)^n ~.
\end{align*}
Let $r_k$ the $k$-th coefficients of $r^{\alpha -n}(t+r)^n $. Then
$$
r_k = r^{\alpha -n}\binom{n}{k}r^{n-k} \leq 2^n r^{\alpha - k}~.
$$
This concludes the proof.
\end{proof}

\subsection{Approximation by shallow Fourier neural networks}

We start by reporting a known result.

\begin{lemma}\label{lemma:burkill}
Let $g:[-\pi,\pi]\to \R$  $2\pi$-periodic with modulus of continuity $\omega$. Then there exists a trigonometric polynomial $q_n(t) = \sum_{k=-n}^n b_ke^{ikt}$, $b_k\in\C$, with real values (i.e. $q_n(t) \in\R$ for all $t\in[-\pi,\pi]$), such that
$$
\sup_{t\in[-\pi,\pi]}\abs*{g(t) - q_n(t)} \leq \frac{2}{\pi}\omega\parr*{\frac{2}{n}}\parq*{2 + \omega(\pi) - \log\omega\parr*{\frac{2}{n}}}~.
$$
Moreover, it holds that
$$
\abs*{b_k} \leq \frac{1}{2\pi}\int_{-\pi}^\pi \abs*{g(t)}\,dt~.
$$
\end{lemma}
\begin{proof}[Proof]
The polinomyal $q_n$ is given by the Fejer sum of the Fourier series of $g$, that is 
$$
q_n(t) = \frac{1}{n}\sum_{j = 0}^{n-1} \sum_{k = -j}^j \hat{g}_k e^{ikt} = \sum_{k = -(n-1)}^{n-1} \frac{n -\abs{k}}{n}\hat{g}_k e^{ikt}
$$
where
$$
\hat{g}_k = \frac{1}{2\pi}\int_{-\pi}^{\pi} g(t) e^{-ikt}\,dt~.
$$
The proof of the upper bound can be found in \citep{burkill1959lectures}, Theorem 18. Finally, notice that $q_n$ is real-valued since
$$
\hat{g}_k e^{ikt} + \hat{g}_{-k} e^{-ikt} = 2\mathrm{Re}\parr*{\hat{g}_k e^{ikt}}
$$
because $\hat{g}_{-k} = \overline{\hat{g}_k}$ since $g$ takes values in $\R$.
\end{proof}
The above result immediately implies a convergence rate for univariate approximation by shallow Fourier networks (that is, with activation $\sigma_1(t) = e^{2\pi i t}$).

\begin{lemma}\label{lemma:approx_fourier_1d}
Let $f:[-r,r] \to \R$ be $L$-Lipschitz. Then there exists a real-valued Fourier shallow  network $q_n(t) = \sum_{k=-n}^n b_ke^{iw_kt}$, $b_k \in \C$, $w_k\in\R$, such that
$$
\sup_{x\in[-r,r]}\abs*{f(x) - q_n(x)} \leq 3\parr*{1 + 2L^2r^2}\frac{\log n}{n}
$$
for any $n \geq 2$. Moreover $q_n$ can be chosen such that $\abs*{w_k} \leq\frac{\pi \abs*{k}}{r}$  and $\abs*{b_k} \leq \norm{f}_\infty$ for any $k\in[-n,n]$.
\end{lemma}
\begin{proof}[Proof]
Assume, w.l.o.g., that $f(r) \leq f(-r)$ (otherwise we can consider $f(-x)$ in place of $f(x)$).
First, we want to transform $f$ into a $2$-pi periodic
function on $[-\pi,\pi]$. To do this we consider $\tilde{g}$ defined as 
$$
\tilde{g}(x) = \begin{cases}
L(x+r) + f(-r) & \text{if } x\in \parq*{-r - \frac{c}{2L}, -r} \\
f(x) & \text{if } x\in [-r,r] \\
L(x-r) + f(r) & \text{if } x\in \parq*{r,r + \frac{c}{2L}}
\end{cases}
$$
where $c = f(-r) - f(r)$. Notice that $\tilde{g}$ is $L$-Lipschitz and $2\parr*{r +\frac{c}{2L}}$-periodic. Finally, let $g:[-\pi,\pi]\to\R$ defined as  
$$
g(x) = \tilde{g}\parr*{\frac{2Lr + c}{2L\pi}x}~.
$$
We have that $g$ is $2\pi$-periodic and $\ell$-Lipschitz for
$$
\ell = \frac{2Lr + c}{2\pi}\leq \frac{2Lr}{\pi}~.
$$
Therefore, we can apply Lemma \ref{lemma:burkill} to $g$. This gives us a (real-valued) trigonometric polynomial $r_n(t) =\sum_{=-n}^n b_ke^{ikt}$ such that
\begin{align*}
\sup_{x\in[-\pi,\pi]}\abs*{g(x) - r_n(x)} & \leq \frac{4\ell}{\pi n}\parq*{2 + \ell \pi - \log\frac{2\ell}{n}} \\
& \leq 3\parr*{1 + 2L^2r^2}\frac{\log n}{n}
\end{align*}
for $n \geq 2$. Since
$$
\sup_{x\in[-r,r]}\abs*{f(x) - r_n\parr*{\frac{L}{\ell}x}} \leq \sup_{x\in\parq*{-r-\frac{c}{2L},r+\frac{c}{2L}}}\abs*{\tilde{g}(x) - r_n\parr*{\frac{L}{\ell}x}} = \sup_{x\in[-\pi,\pi]}\abs*{g(x) - r_n(x)}
$$
the thesis follows.
\end{proof}
To conclude we make some remarks about shallow Fourier networks. Note that a generic shallow Fourier network $f_N$ with $N$ units can be represented as
\begin{equation}\label{eq:fourier_nets}
f(\bx) = \sum_{k=1}^N u_ke^{i\bw_k^T\bx}~.
\end{equation}
Indeed we have that
$$
\sum_{k=1}^N u_ke^{i(\bw_k^T\bx +b_k) } + b = \sum_{k=1}^N \parr*{u_ke^{ib_k}}e^{i\bw_k^T\bx  } + b\cdot e^{i\bzero^T\bx}
$$
for any $b$, $b_k \in\C$. Let $\mathcal{F}_N^f$ be the space of networks as in equation \eqref{eq:fourier_nets}. Notice that a universal approximation theorem holds for shallow Fourier networks as well. This is because the universal approximation theorem holds for shallow networks with activation $\sigma(t) = \cos(t)$ and since $\cos(t) = \parr*{e^{it} + e^{-it}} / 2$, the thesis follows. Finally, the following lemma will be used in the proof of Theorem \ref{theo:approx_shallow}.

\begin{lemma}\label{lemma:fnets_prod}
If $f$ is a (real-valued) shallow Fourier neural network, then so is $f^k$, for $k$ non-negative integer. Moreover, if $f$ has $n$ units, then the number of units of $f^k$ is upper bounded by 
$$
\binom{n+k-1}{k}~.
$$
\end{lemma}
\begin{proof}[Proof]
Let $f(\bx) = \sum_{j=1}^n u_je^{i\bw_j^T\bx}$ be a shallow Fourier neural network. Then, by the multinomial formula, we have that
\begin{align*}
f^k(\bx) & = \parr*{\sum_{j=1}^n u_je^{i\bw_j^T\bx}}^k = \sum_{p_1+\cdots+p_n = k}\binom{k}{p_1,\dots,p_n} \prod_{j=1}^n\parr*{u_j^{p_j}\parr*{e^{i\bw_j^T\bx}}^{p_j}} \\
& = \sum_{p_1+\cdots+p_n = k}\binom{k}{p_1,\dots,p_n} \parr*{\prod_{j=1}^n u_j^{p_j} } e^{i\parr*{\sum_{j=1}^n p_j\bw_j}^T\bx}~.
\end{align*}
Clearly, if $f$ is real-valued, so is $f^k$. Finally notice that by the formula above, the number of units of $f^k$ is upper bounded by $\abs*{\bra*{(p_1,\dots,p_n)\st p_1+\cdots+p_n = k}}$.
\end{proof}

\subsection{\texorpdfstring{$\mathrm{poly}(d)$}{poly(d)} upper bounds for two-hidden-layers networks}\label{app:poly(d)ub-3layers}

Consider a two-hidden-layers neural network $f$ defined as
$$
f: \bx\in\R^d \mapsto \bgamma^T \bg\parr*{\bW^T \bh\parr*{\bU^T\bx}} \in \C  ~,  
$$
where $\bh:\R^p \to \R^p$ and $\bg:\R^o \to \R^o$ are, respectively, component-wise $1$-Lipschitz and $(1,\alpha)$-Holder activation functions, and $\bU 
\in\R^{d\times p}$, $\bW 
\in \R^{p\times o}$, $\bgamma\in\C^o$. We wish to  approximate $f$ with a one-hidden-layer neural network with a given activation $\sigma$ satisfying Assumption \ref{ass:activation}.2, for some constant $\nu_\sigma >0$. We start by proving a result for approximation by shallow Fourier networks at a  $\mathrm{poly}(d)$ rate.  

\begin{proposition}\label{prop:app_fourier_nn}
Let $K\subset \R^d$ be a compact set. There exist $f_N \in \mathcal{F}_N^f$ such that
\begin{equation}\label{eq:inf_fourier}
\norm*{f - f_N^f}_{K,\infty} \leq \epsilon 
\end{equation}
with
$$
f_N^f(\bx) = \sum_{\nu=1}^N b_\nu e^{i\bv_\nu^T\bx}~,
$$
for 
$$
N = (2np +1)^m
$$
with
$$
n  = \ceil*{ \frac{9\cdot 4^{\frac{1}{\alpha}}\norm{\bgamma}_1^2\norm{\bW}_\infty^2(1+2C^2)^2}{\epsilon^{\frac{2}{\alpha}}} }
\quad 
\text{and}
\quad
m = \ceil*{  \frac{2\cdot 16^{\frac{1}{\alpha}}}{\epsilon^{1 + \frac{2}{\alpha}}}\norm{\bgamma}_1^{\frac{1}{\alpha}}\parr*{\parr*{\frac{\epsilon}{2\norm{\bgamma}_1}}^{\frac{1}{\alpha}} + M}^{\alpha} }~,
$$
where we denoted
$$
C = \sup_{x\in K}\norm{\bU^T\bx}_\infty \quad \text{and} \quad M = \sup_{x\in K} \norm*{\bW^Th\parr*{\bU^T\bx}}_\infty ~.
$$
Moreover $f_N^f$ can be chosen such that it holds
\begin{equation}
\label{eq:fourier_app_coeffs}
\sup_{\bx\in K}\abs*{\bv_\nu^T\bx} \leq  \pi m n
\quad 
\text{and}
\quad
\abs*{b_\nu} \leq 
2\norm{\bgamma}_1\parq*{1 + \parr*{\parr*{\frac{\epsilon}{2\norm{\bgamma}_1}}^{\frac{1}{\alpha}} + M}^\alpha}\parr*{4 n p H \norm{\bW}_{F,\infty}}^m
\end{equation}
where $H = \sup_{\bx \in [-C,C]^d}\norm{h(\bx)}_\infty$.
\end{proposition}
\begin{proof}[Proof]
Let $q_n^j$ given by Lemma \ref{lemma:approx_fourier_1d} to approximate $h_j$ over $[-C,C]$ and
$$
q_k^{(n)}(\bx) = \sum_{j=1}^p w_{k,j}q_n^j(\bu_j^T \bx )
$$
for $k \in [o]$.
We have that
\begin{align*}
\abs*{q^{(n)}_k(\bx) - \bw_k^Th\parr*{\bU^T\bx}} & \leq \sum_{j=1}^p\abs*{w_{k,j}}\abs*{q_n^j(\bu_j^T\bx) - h_j(\bu_j^T\bx)} \\
& \leq   3\norm{\bW}_\infty\parr*{1 + 2C^2}\frac{\log n}{n} \doteq \norm{\bW}_\infty(1+2C^2)\epsilon_n
\end{align*}
for $\bx \in K$. It holds that $q_k^{(n)}$ is a real-valued shallow Fourier network with $(2n-1)p$ terms and first layers weights given by $\frac{\pi k}{C}\bu_j$ for $k \in [-(n-1),n-1]$. Moreover, it holds that
$$
\abs*{q_k^{(n)}(\bx)} \leq \abs*{q_k^{(n)}(\bx) - \bw_k^Th\parr*{\bU^T\bx}} + \abs*{\bw_k^Th\parr*{\bU^T\bx}} \leq \norm{\bW}_\infty(1+2C^2)\epsilon_n + M \doteq L ~.
$$
Let $p_m^k(t) = \sum_{h=0}^m \beta_h^k t^h$ given by Corollary 3 to approximate $g_k$ over  the interval $[-L,L]$ and $\epsilon_m$ the relative error. Let then
$$
f_{n,m}(\bx) = \sum_{k=1}^o \gamma_k p^k_m(q_k^n(\bx)) ~.
$$
It holds that
\begin{align*}
\abs*{f(\bx) - f_{n,m}(\bx)} & \leq
\sum_{k=1}^o \abs*{\gamma_k}\abs*{g_k(\bw_k^Th(\bU^T\bx)) - p^k_m(q_k^{(n)}(\bx))}
\\
& \leq \sum_{k=1}^o \abs*{\gamma_k}\abs*{g_k(\bw_k^Th(\bU^T\bx)) - g_k(q_k^n(\bx))} + \sum_{k=1}^o \abs*{\gamma_k}\abs*{g_k(q_k^{(n)}(\bx)) - p^k_m(q_k^{(n)}(\bx))}
\\
& \leq \norm{\bgamma}_1\sup_{k \in [o]}\abs*{\bw_k^Th(\bU^T\bx) - q_k^{(n)}(\bx)}^{\alpha} + \norm{\bgamma}_1\epsilon_m \\
& \leq \norm{\bgamma}_1\norm{\bW}_\infty^{\alpha}(1+2C^2)^{\alpha}\epsilon_n^{\alpha} + \norm*{\bgamma}_1\epsilon_m ~.
\end{align*}
It holds that
$$
\norm{\bgamma}_1\norm{\bW}_\infty^\alpha(1 + 2C^2)^\alpha\epsilon_n^\alpha \leq \frac{\epsilon}{2}
$$
as long as
\begin{equation}\label{eq:n_eps_w}
n \geq  \frac{9 \cdot 4^{\frac{1}{\alpha}}\norm{\bgamma}_1^2\norm{\bW}_\infty^2\parr*{1 + 2C^2}^2}{\epsilon^{\frac{2}{\alpha}}} ~.
\end{equation}
Similarly
$$
\norm*{\bgamma}_1\epsilon_m \leq \frac{\epsilon}{2}
$$
as long as 
$$
m\geq L \parr*{\frac{12\norm*{\bgamma}_1}{\epsilon}}^{\frac{1}{\alpha}} = \parr*{\frac{12\norm*{\bgamma}_1}{\epsilon}}^{\frac{1}{\alpha}} \parq*{\norm{\bW}_\infty(1+2C^2)\epsilon_n + M}~.
$$
Moreover, by Lemma \ref{lemma:safran}, $p_m^k(t) = \sum_{h=0}^m \beta_h^k t^h$ can be chosen with
$$
m \geq \frac{2\cdot16^{\frac{1}{\alpha}}}{\epsilon^{1+\frac{2}{\alpha}}}\norm*{\bgamma}_1^{\frac{1}{\alpha}}L^\alpha = \frac{2\cdot16^{\frac{1}{\alpha}}}{\epsilon^{1+\frac{2}{\alpha}}}\norm*{\bgamma}_1^{\frac{1}{\alpha}}\parq*{\norm{\bW}_\infty(1+2C^2)\epsilon_n + M}^{\alpha}
$$
such that its coefficients $\beta_h^k$, $k \in [m]$, are bounded by
\begin{align*}
\abs*{\beta_k} & \leq \max\bra*{2^m L^{\alpha-k} , L^\alpha + \abs*{g(0)}} \leq 2^m (1 + L^\alpha) + \abs*{g(0)} \\
& = 2^m \parr*{ 1 + \parq*{\norm{\bW}_\infty(1+2C^2)\epsilon_n + M}^{\alpha} } +\abs*{g(0)}~.
\end{align*}
Notice that we can assume $g(0) = 0$ without loss of generality. Therefore
\begin{equation}
\sup_{x\in K}\abs*{f(\bx) - f_{n,m}(\bx)} \leq \epsilon
\label{eq:error_holds}
\end{equation}
as long as \eqref{eq:n_eps_w} holds and
\begin{equation}\label{eq:m_eps_w}
m\geq \parr*{\frac{12\norm*{\bgamma}_1}{\epsilon}}^{\frac{1}{\alpha}} \parq*{\parr*{\frac{\epsilon}{2\norm{\bgamma}_1}}^{\frac{1}{\alpha}} + M} = 6^{\frac{1}{\alpha}}\parr*{1  + M\parr*{\frac{2\norm{\bgamma}_1}{\epsilon}}^{\frac{1}{\alpha}}}~.
\end{equation}
If we further assume that 
\begin{equation}\label{eq:m_eps_w_control}
m \geq \frac{2\cdot16^\frac{1}{\alpha}}{\epsilon^{1 + \frac{2}{\alpha}}}\norm*{\bgamma}_1^{\frac{1}{\alpha}}\parq*{ \parr*{\frac{\epsilon}{2\norm{\bgamma}_1}}^{\frac{1}{\alpha}} + M}^\alpha
\end{equation}
we can also assume that
$$
\abs*{\beta_h^k} \leq 2^{1+ \frac{2\cdot16^\frac{1}{\alpha}}{\epsilon^{1 + \frac{2}{\alpha}}}\norm{\bgamma}_1^{\frac{1}{\alpha}}\parq*{\parr*{\frac{\epsilon}{2\norm{\bgamma}_1}}^{\frac{1}{\alpha}} + M}^\alpha} \parr*{ 1+ \parq*{\parr*{\frac{\epsilon}{2\norm{\bgamma}_1}}^{\frac{1}{\alpha}} + M}^\alpha} 
$$
for $k \in [m]$.
Finally, notice that, by Lemma \ref{lemma:fnets_prod}, $f_{n,m}$ is a shallow Fourier neural network with number of units upper bounded by
\begin{align*}
N &= \sum_{k=0}^m\binom{(2n-1)p + k - 1}{k} = \binom{(2n-1)p + m}{m} \\
& = \frac{1}{m!}((2n-1)p + k +m)\cdots ((2n-1)p +1) \\
& \leq ((2n-1)p + 1)^m~.
\end{align*}
Therefore, it holds that
$$
\inf_{f_N \in\mathcal{F}_N^f}\sup_{\bx\in K} \abs{f(\bx) - f_N(\bx)} \leq \epsilon 
$$
as long as 
$$
N \geq \parr*{2n p + 1}^{m}
$$
with $n$ and $m$ given by \eqref{eq:n_eps_w} and \eqref{eq:m_eps_w} respectively.
Finally, notice that the first layer weights of $f_{n,m}$ are given by 
$$
\sum_{j=1}^p\sum_{k=-(n-1)}^{n-1}  s_{k,j}\frac{\pi k}{C} u_j 
$$
over all non-negative integers $s_{k,j}$ such that
$\sum_{j=1}^p\sum_{k=-(n-1)}^{n-1} s_{k,j} \leq m$. Therefore, if 
$$
f_{n,m}(\bx) = \sum_{\nu=1}^N b_\nu e^{i\bv_\nu^T\bx}~,
$$ 
then
$$
\abs*{\bv_\nu^T \bx} \leq m \frac{\pi (n-1)}{C}\max_{j \in [p]}\abs{\bu_j^T\bx} \leq m n \pi ~.
$$
On the other hand, the coefficients $b_k$ have the form
$$
b_\nu = \binom{h}{s} \sum_{k = 1}^o \gamma_k \beta_h^k \parr*{w_{k,j}(q_n^j)_l}^{s_{l,j}}
$$
for all non-negative integers $s = \parr*{s_{l,j}}_{l,j}$ such that
$\sum_{j=1}^p\sum_{l=-(n-1)}^{n-1} s_{l,j} = h \leq m$, where $(q_n^j)_l$ denotes the $l$-th coefficients of $q_n^j$. By Lemma \ref{lemma:burkill}, we know that
$$
\abs*{(q_n^j)_l} \leq \sup_{t \in [-C,C]}\abs*{h_j(t)}~.
$$
Therefore
\begin{align*}
\abs*{b_\nu} & \leq 
\parr*{(2n-1)p}^h \sup_{t\in[-C,C]}\abs*{h_j(t)}^{s_{l,j}} \sum_{k=1}^o \abs*{\gamma_k}\abs{\beta_h^k}\abs*{w_{k,j}}^{s_{l,j}}
\\
& \leq \parq*{(2n-1)p\, H \, \norm{\bW}_{F,\infty}}^m\norm{\bgamma}_1 \norm{\bbeta}_{F,\infty}~.
\end{align*}
This concludes the proof.
\end{proof}
We can now conclude with a detailed version of Theorem \ref{theo:approx_shallow}.

\begin{theorem}\label{theo:poly(d)_details}
Let $K$ be a compact set and $$
C = \sup_{\bx\in K}\norm{\bU^T\bx}_\infty \,, \quad M = \sup_{\bx\in K} \norm*{\bW^T\bh\parr*{\bU^T\bx}}_\infty \quad \text{and} \quad H = \sup_{\bx \in [-C,C]^d}\norm{\bh(\bx)}_\infty~.
$$
It holds that
\begin{equation}\label{eq:inf_generic}
\inf_{f_N^\sigma \in\mathcal{F}_N^\sigma} \norm*{f(\bx) - f_N^\sigma(\bx)}_{K,\infty} \leq \epsilon 
\end{equation}
for some
$$
N \leq \frac{16 \pi \nu_\sigma}{\epsilon}  \norm{\bgamma}_1 m n (4np + 1)^{2m} \parr*{H\norm{\bW}_{F,\infty}}^m \parq*{1 + \parr*{\parr*{\frac{\epsilon}{2\norm{\bgamma}_1}}^{\frac{1}{\alpha}} + M}^\alpha}~,
$$
where
$$
n = \frac{9\cdot 4^{\frac{1}{\alpha}}\norm{\bgamma}_1^2\norm{\bW}_\infty^2(1+2C^2)^2}{\epsilon^{\frac{2}{\alpha}}}
\quad 
\text{and}
\quad
m = \frac{2\cdot 16^{\frac{1}{\alpha}}}{\epsilon^{1 + \frac{2}{\alpha}}}\norm{\bgamma}_1^{\frac{1}{\alpha}}\parr*{\parr*{\frac{\epsilon}{2\norm{\bgamma}_1}}^{\frac{1}{\alpha}} + M}^{\alpha}~.
$$
Moreover, it is possible to choose $f^\sigma_N$ attaining \eqref{eq:inf_generic} with $m_\infty\parr*{f^\sigma_N}$ satisfying a bound similar to the one on $N$, for example $m_\infty\parr*{f^\sigma_N} \leq (1+N^2)$. 
\end{theorem}
\begin{proof}[Proof]
Let $f_N$ given by Proposition \ref{prop:app_fourier_nn} such that 
$$
\sup_{\bx\in K}\abs*{f(\bx) - f_N(\bx) }\leq \frac{\epsilon}{2}~.
$$
We know that
$$
f_N(\bx) = \sum_{k=1}^N b_ke^{i\bv_k^T\bx} = f_N^c(\bx) + i f_N^s(\bx)
$$
where
$$
f_N^c(\bx) = \sum_{k=1}^N b_k\cos\parr{\bv_k^T \bx} \quad \text{and} \quad f_N^s(\bx) = \sum_{k=1}^N b_k\sin\parr{\bv_k^T \bx}
$$ 
and $\abs*{b_k} \leq B$ and $\abs*{\bv_k^T\bx} \leq V$ for $\bx\in K$, where $B$ and $V$ are given by \eqref{eq:fourier_app_coeffs}. Using the assumption on $\sigma$, we know that, for each $k\in[N]$, there exist shallow networks $f_k^c$ and $f_k^s$ with activation $\sigma$ and number of units
$$
n \leq c_\sigma\frac{4VBN}{\epsilon}
$$
such that
$$
\sup_{\bx\in K}\abs*{f_k^c(\bx) - \cos(\bv_k^T\bx)} \leq \frac{\epsilon}{4NB} \quad \text{and} \quad \sup_{\bx\in K}\abs*{f_k^s(\bx) - \sin(\bv_k^T\bx)} \leq \frac{\epsilon}{4NB}~.
$$
Letting $f_\NN(\bx) = \sum_{k=1}^N b_k f_k^c(\bx) + i \sum_{k=1}^N b_k f_k^s(\bx) $ it holds that
\begin{align*}
\sup_{\bx\in K}\abs*{ f_\NN(\bx) - f_N(\bx) } & \leq \sup_{\bx\in K}\abs*{\sum_{k=1}^Nb_k\parr*{f^c_k(\bx) - \cos(\bw_k^T\bx)}} + \sup_{\bx\in K}\abs*{\sum_{k=1}^Nb_k\parr*{f^s_k(\bx) - \sin(\bw_k^T\bx)}} \\
& \leq \sum_{k=1}^N\abs*{b_k}\sup_{\bx\in K}\abs*{f^c_k(\bx) - \cos(\bw_k^T\bx)} + \sum_{k=1}^N\abs*{b_k}\sup_{\bx\in K}\abs*{f^s_k(\bx) - \sin(\bw_k^T\bx)} \\
& \leq NB\frac{\epsilon}{4NB} + NB\frac{\epsilon}{4NB} = \frac{\epsilon}{2}
\end{align*}
which implies that
$$
\sup_{\bx\in K}\abs*{ f_\NN(\bx) - f(\bx) } \leq \epsilon~.
$$
Moreover notice that we can assume that all second layer weights of $f_\NN$ are real; indeed, if this is not the case, one can replace them by the real part, and upper bound above can only decrease. Finally, we have that the number of units of $f_\NN$ is given by
$$
\NN \leq  \frac{8c_\sigma}{\epsilon} \cdot V \cdot B \cdot N~.
$$
Applying Proposition \ref{prop:app_fourier_nn} concludes the proof.
\end{proof}

\subsection{Proofs of special cases}\label{proofs:special_cases}

\subsubsection{Radial functions}

Let $f(\bx) = \varphi(\norm{\bx})$ with $\varphi$ $1$-Lipschitz. Then it holds that $f(\bx) = g(\bone^T \bh(\bx))$ where $g(t) = \varphi(\sqrt{t})$ and $\bh:\R^d \to \R^d$ is defined as $h_i(\bx) = x_i^2$. Clearly, $\sup_{\bx\in B_{1,2}^d}\norm{\bx}_ \infty = 1$, $\sup_{\bx\in B^d_{1,2}}\abs*{\bone^T\bh(\bx)} = \sup_{\bx\in B^d_{1,2}} \norm{\bx}^2 = 1$ and $\sup_{\bx \in [-1,1]^d}\norm{\bh(\bx)}_\infty = \sup_{x\in[-1,1]}\abs*{x}^2 = 1$. Moreover, $\norm{\bone}_1 = d$ and $g$ is $(1,1/2)$-Holder. Then, by applying Theorem \ref{theo:poly(d)_details}, we get the following.
\begin{corollary}[Radial functions]
It holds that
$$
\inf_{f_N^\sigma \in\mathcal{F}_N^\sigma} \norm{f^\sigma_N - f}_{ B_{1,2}^d,\infty} \leq \epsilon
$$
for some
$$
N \leq \nu_\sigma \alpha \cdot d^2\cdot \frac{\parr{4 + \epsilon }^2}{\epsilon^{10}} \parr*{ \alpha \frac{d^3}{\epsilon^4} + 1}^{ \frac{\alpha}{\epsilon^5}\parr*{2 + \epsilon} }
$$
where $\alpha>0$ is a numerical constant.
\end{corollary}

\subsubsection{Shallow approximation of piece-wise oscillatory functions}

Consider $f_{\bw,\bU} : \bx\in\R^d \mapsto e^{i\bw^T\parr*{\bU\bx}_+}$ for some $\bw\in\R^p$, $\bU \in \R^{p\times d}$. Then Theorem \ref{theo:poly(d)_details} implies the following. 

\begin{corollary}[Approximation of \eqref{eq:defig} by shallow networks]
It holds that
$$
\inf_{f_N^\sigma\in\mathcal{F}_N^\sigma} \norm*{f_{\bw,\bU} - f_N^\sigma}_{B_{r,p}^d, \infty} \leq \epsilon
$$
for some
$$
N \leq \frac{\nu_\sigma \beta}{\epsilon^6} \cdot \parr*{2 + \epsilon + 2r\norm{\bw}_1\norm{\bU}_{p,\infty}}^2 \cdot 
\parq*{r \norm{\bw}_\infty \norm{\bU}_{p,\infty}
\parr*{\frac{4p\beta}{\epsilon^2}+1}^{2}}^{\frac{\alpha}{\epsilon^2}\parr*{\epsilon + 2r \norm{\bw}_1\norm{\bU}_{p,\infty}}}
$$
where $\beta = \alpha\norm{\bw}_1^2 \cdot \parr*{1 + 2r^2 \norm{\bU}_{p,\infty}^2}^2$
and $\alpha$ is a numerical constant.
\end{corollary}

\subsubsection{Approximation bounds under the Gaussian metric}\label{app:proof_agussian_ub}

For sake of simplicity in this section we consider approximation bounds for the function of interest
$$
f_{\bw,\bU} :\bx\in\R^d \mapsto e^{i\bw^T(\bU\bx)_+}
$$
for some $\bw \in\R^p$, $\bU = [\bu_1|\cdots|\bu_p]^T\in\R^{p\times d}$. Notice that the following results can be naturally extended to any three-layer network target. We are interested in upper bounding the error
$$
\inf_{f_N\in\mathcal{F}_N^f} \parr*{ \E\abs*{f_{\bw,\bU}(\bX) - f_N(\bX)}^2 }^{\frac{1}{2}}
$$
where the expectation is taken over $\bX \sim N(\bzero,\sigma^2\bI)$. For sake of simplicity of notation, we denote
$$
\norm{f - g}_{\sigma,2} \doteq \parr*{ \E\abs*{f(\bX)-g(\bX)}^2 }^{\frac{1}{2}}~.
$$
It is a well known fact that Gaussian vectors concentrates in a ball of radius $\sqrt{d}$. We recall a quantitative version of this fact in the following.
\begin{lemma}\label{lemma:delta_error}
Let $\bX \sim  N(\bzero,\sigma^2\bI)$ a $d$-dimensional Gaussian vector. Then it holds that
$$
P\bra*{\norm{\bX}_2 \geq \sigma \sqrt{d} + t} \leq e^{-\frac{t^2}{2\sigma^2}}~.
$$
\end{lemma}
Thanks to Proposition \ref{prop:app_fourier_nn}, the following holds.
\begin{lemma}\label{lemma:gaussian_ub_fourier_nets}
Let $r >0$. Then it holds that
\begin{equation}\label{eq:delta_error}
\inf_{f_N\in\mathcal{F}_N^f}\norm*{f_N - f_{\bw,\bU}}_{B_{r,2}^d,\infty} \leq \delta
\end{equation}
as long as 
$$
N \geq (2np +1)^m
$$
where 
$$
n = \frac{36}{\delta^2}\norm{\bw}_1^2\parr*{1+r^2\norm{\bU}_{2,\infty}^2}^2 \quad \text{and} \quad m \geq \frac{16}{\delta^3}\parr*{\delta + 2r \norm{\bw}_1\norm{\bU}_{2,\infty}}~.
$$
Moreover, under the same assumption, we can also assume that the function $f_N$ that satisfies \eqref{eq:delta_error} also satisfies
$$
\norm{f_N}_\infty \leq N\parr*{2 + \delta + 2r\norm{\bw}_1\norm{\bU}_{2,\infty}}\parr*{4npr\norm{\bw}_\infty \norm{\bU}_{2,\infty}}^m~.
$$
\end{lemma}
Thanks to these two lemmas, the following proposition follows.

\begin{proposition}\label{prop:ub_gaussian}
Let $\sigma= d^{-1/2}$ and assume that $\norm{\bU}_{2,\infty} \leq 1$. Then it holds
\begin{equation}\label{eq:gauss_norm_eps}
\inf_{f_N \in\mathcal{F}_N^f}\norm{f_N-f_{\bw,\bU}}_{\sigma,2} \leq \epsilon
\end{equation}
as long as 
$$
N \geq\parq*{K p \parr*{1 +\frac{1}{\epsilon^s}}\parr*{1 + \norm{\bw}_1^s}}^{K \parr*{1 + \parr*{\frac{\log p}{d}}^s} \parr*{1 + \frac{1}{\epsilon^s}}\parr*{1 + \norm{\bw}_1^s}}
$$
where $K >0$ and $s \geq 1$ are some numerical constant.
\end{proposition}
\begin{proof}[Proof]
Let $c = \norm{\bw}_1$. 
First, notice that $\norm{f_{\bw,\bU}}_\infty = 1$. Let $\chi_r(\bx) = \mathbbm{1}\bra*{\norm{\bx}_2 \leq r}$ and $f_N$ given by Lemma \ref{lemma:gaussian_ub_fourier_nets} for a certain $\delta>0$. Then it holds that
\begin{align*}
\norm{f_N - f_{\bw,\bU}}_{\sigma,2} & \leq \norm{(f_N-f_{\bw,\bU})(1 - \chi_r)}_{\sigma,2} + \norm{(f_N-f_{\bw,\bU})\chi_r}_{\sigma,2} \\
& \leq \norm*{f_N - f_{\bw,\bU}}_{B_{r,2}^d,\infty} + P\parr*{\norm{\bx}_2 > r}(\norm{f_N}_\infty + \norm{f_{\bw,\bU}}_\infty)~.
\end{align*}
If $r = 1 + t$ for $t>0$, it follows
\begin{align*}
\norm{f_N - f_{\bw,\bU}}_{2,\sigma} & \leq \delta + e^{-\frac{dt^2}{2}}(1 + \norm{f_N}_\infty)    
\end{align*}
as long as 
$$
N \geq \parr*{\frac{72 p }{\delta^2}c^2\parr*{1 + r^2}^2 + 1}^{\frac{1}{\delta^3}\parr*{\delta + 2 rc}}~.
$$
Moreover, one can assume
\begin{align*}
\norm{f_N}_\infty & \leq (2+\delta + 2 r c) \parr*{\frac{72 p }{\delta^2}c^2\parr*{1 + r^2}^2 + 1}^{\frac{16}{\delta^3}\parr*{\delta + 2 r c}}
\parr*{ 
144\frac{pr}{\delta^2}c^3\parr*{1 +  r^2}^2
}^{\frac{16}{\delta^3}\parr*{\delta + 2r}} \\
& \leq (2 + \delta + 2r\omega)\parr*{144 \frac{p}{\delta^2}\omega^3r(1+r^2)^2 +1}^{\frac{32}{\delta^3}\parr*{\delta + 2r\omega}}
\end{align*}
where $\omega = \max\parr*{1,c}$.
Let $ \delta = \frac{\epsilon}{2}$. If $t\geq 1$, it holds that
\begin{align*}
\norm{f_N}_\infty & \leq (4\omega + \epsilon + 2\omega t)\parr*{
576\frac{p}{\epsilon^2}\omega^3(1+t)\parr*{1+(1+t)^2}^2 + 1
}^{\frac{256}{\epsilon^3}(\epsilon + 2\omega + 2\omega t)} \\
& \leq K(\epsilon + \omega + \omega t)\parr*{
K\frac{p}{\epsilon^2}\omega^2 t^5 + 1
}^{\frac{K}{\epsilon^3}(\epsilon + \omega + \omega t)} ~.
\end{align*}
In the equation above above and in the following, $K$ denotes a (large enough) numerical constant. 
Therefore
\begin{equation}\label{eq:guass_decay_part}
e^{-\frac{dt^2}{2}}\parr*{ 1 + \norm{f_N}_\infty} \leq \frac{\epsilon}{2}
\end{equation}
as long as 
$$
\frac{dt^2}{2} - \log\parr*{1 + K(\epsilon + \omega + \omega t)\parr*{
K\frac{p}{\epsilon^2}\omega^2 t^5 + 1
}^{\frac{K}{\epsilon^3}(\epsilon + \omega + \omega t)}} + \log \frac{\epsilon}{2} \geq 0~.
$$
Since $\log(1+Cs^\alpha) \leq \log(1 + C) + \alpha\log(s)$ if $s\geq 1$, $C>0$ and $\alpha >0$, the above is implied by 
$$
\frac{dt^2}{2} - \log\parr*{1 + K(\epsilon + \omega + \omega t)} -{\frac{K}{\epsilon^3}(\epsilon + \omega + \omega t)}\log\parr*{
K\frac{p}{\epsilon^2}\omega^2 t^5 + 1
} + \log \frac{\epsilon}{2} \geq 0~.
$$
Since 
$$
\log\parr*{1 + K(\epsilon + \omega + \omega t)} \leq K(\epsilon + \omega + \omega t)
$$
and 
$$
\log\parr*{
K\frac{p}{\epsilon^2}\omega^2 t^5 + 1
} \leq \log\parr*{1 + K \frac{p\omega^2}{\epsilon^2}} + 5\log t \leq \log\parr*{1 + K \frac{p\omega^2}{\epsilon^2}} + 5\sqrt{t}
$$
equation \eqref{eq:guass_decay_part} holds if
$$
\frac{dt^2}{2} - \alpha - \beta t^{1/2} - \gamma t - \eta t^{3/2} \geq 0 
$$
where 
\begin{align*}
\alpha & = K ( \epsilon + \omega ) + \frac{K}{\epsilon^3}(\epsilon + \omega)\log\parr*{1 + K \frac{p\omega^2}{\epsilon^2}} - \log\frac{\epsilon}{2} > 0 ~,\\
\beta & = \frac{K}{\epsilon^3}(\epsilon + \omega) > 0~,\\
\gamma & =  K \omega t + \frac{K}{\epsilon^3}\omega \log\parr*{1 + K \frac{p\omega^2}{\epsilon^2}}  > 0 ~,\\
\eta & = \frac{K}{\epsilon^3}\omega t > 0~.
\end{align*}
It follows that eq. \eqref{eq:guass_decay_part} holds if 
$$
t \geq 1 + 4\parr*{\frac{\alpha + \beta +\gamma + \eta}{d}}^2~.
$$
It follows that the error bound \eqref{eq:gauss_norm_eps} holds as long as
\begin{align*}
N & \geq  \parr*{\frac{K p }{\epsilon^2}(1+c)^2\parr*{1 + 4\parr*{\frac{\alpha + \beta +\gamma + \eta}{d}}^2}^4 + 1}^{\frac{K}{\epsilon^3}\parr*{\epsilon + c\parr*{1 + \parr*{\frac{\alpha + \beta +\gamma + \eta}{d}}^2}}} ~.
\end{align*}
The thesis follows.
\end{proof}

\subsection{Extension to generic \texorpdfstring{$L$}{L}-layers networks}\label{app:multi-layer-ub-polyd}

The results presented in the previous section can be generalized to hold for approximating generic multi-layer neural networks. In this section we present an analogous result to Theorem \ref{theo:approx_shallow} for this more general case. 
Consider a multi-layer neural network $f$ defined as
$$
f : \bx \in \R^d \to x^{(L)}(\bx) \in \C
$$
where $x^{(L)}$ is defined by recursion by $\bx^{(0)}(\bx) = \bx$,
\begin{align*}
\bx^{(k)}(\bx) = \bsigma^{(k)}(\bA^{(k)} \bx^{(k-1)}(\bx) ) ~ \text{for } k \in [L] \quad\text{and }\quad
x^{(L+1)}(\bx)  = \parq*{ \ba^{(L+1)} }^T \bx^{(L)}(\bx)~,
\end{align*}
where $\bA^{(k)} = \parq{\ba_{1}^{(k)} |\cdots | \ba_{d_k}^{(k)}}^T \in \R^{d_{k}\times d_{k-1}}$ for $k \in [L]$ (with $d_0 = d$), 
$\ba^{(L+1)} \in \mathbb{C}^{d_L}$ and $\bsigma^{(k)} : \R^{d_{k}} \to \R^{d_{k}}$ are $\frac{1}{6}$-Lipschitz component-wise activation functions  and verify $\bsigma^{k}(\bzero) = \bzero$ for $k \in [L]$. In the following we also assume that  $\norm{\bA^{(k)}}_\infty \leq 1$ for $k \in [L]$ and $\norm{\ba_{L+1}}_1 \leq 1$. Note that these assumption can easily be relaxed, but we adopt them here for sake of simplicity. 

\begin{proposition}\label{prop:multi-layer-ub-polyd}
Let $f$ as above. It holds that
\begin{equation}\label{eq:error_multi_layer}
\inf_{f_N \in \mathcal{F}_N^f}\norm*{f - f_N}_{B_{1,\infty}^d,\infty} \leq \epsilon    
\end{equation}
as long as 
$$
N \geq \parr*{2^L C \parr*{1 + \frac{1}{\epsilon^2}} d_1 }^{C L \parr*{ 1 + \frac{1}{\epsilon}}^{L-1}}
$$
where $C$ is a numerical constant.
\end{proposition}

Before proving the above proposition, we prove two preliminary lemmas.

\begin{lemma}\label{lemma:Fourier_nets_prod_multi}
Let $\mathcal{W} = \bra*{\bw_\ell}_{\ell \in [K]} \subset \R^d$ and $\bh:\R^d \to \R^p$ such that $h_j$ is a shallow Fourier neural networks with first layer weights given by $\mathcal{W}$, for all $j \in [p]$. Consider $\bq: \R^p \to \R^m$ of the form
$$
\bq(\bx) = \bB \bsigma(\bx)
$$
where $\bsigma:\R^p \to \R^p$ is a component-wise polynomial activation function of degree at most $D$ and $\bB \in \C^{m\times p}$.
Then there exists $\mathcal{V} \subset \R^d $ finite such that $\bfun \doteq \bq \circ \bh$ is such that $f_j$ is a Fourier neural nets with first layer weights given by $\mathcal{V}$ for each $j \in [p]$ and such that 
$$
\abs*{\mathcal{V}} \leq (2K)^D~.
$$
\end{lemma}
\begin{proof}[Proof]
The functions $f_j$ have the form
$$
f_j(\bx) = \sum_{k=1}^p b_{jk} \sum_{l = 0}^D \alpha_{k,l}(h_k(\bx))^l = \sum_{k=1}^p b_{jk} \sum_{l = 0}^D \alpha_{k,l}\parr*{ \sum_{\nu = 1}^K \beta_{k,\nu} e^{i \bw_\nu^T \bx} }^l ~.
$$
By Lemma \ref{lemma:fnets_prod}, we see that each $f_j$ is a Fourier neural network with the same set of first layer weights of size at most
\begin{align*}
\sum_{l=0}^D\binom{K + l  - 1}{l}   & = \binom{K + D}{D}   \leq (K+1)^D \leq (2K)^D~.
\end{align*}
This concludes the proof.
\end{proof}

\begin{lemma}\label{lemma:poly_upper_bound_multi_layer}
Consider the same assumption as Proposition \ref{prop:multi-layer-ub-polyd}. Then, there exists a polynomial $$
f_{N_1,\dots,N_L}:\bx\in \R^d \to y^{(L+1)}(\bx) \in \C
$$ 
given by the recursion $\by^{(0)}(\bx) = \bx$, 
\begin{align*}
\by^{(k)}(\bx) &= \bp_{N_k}^k(\bA^{(k)}\by^{(k-1)}(\bx)) \quad \text{for } k \in [L] \\
y^{(L+1)}(\bx) & = \parq*{ \ba^{(L+1)}}^T \by^{(L)}(\bx)
\end{align*}
where $\bp_{N_k}^k$ are component-wise polynomial activation functions of degree $N_k$, such that
\begin{equation}\label{eq:poly_multi_layer_lemma}
\norm{ f - f_{N_1,\dots,N_L} }_{B_{1,\infty}^d,\infty} \leq \epsilon
\end{equation}
as long as $N_k \geq \frac{L}{\epsilon} + (L-1)$ for $k\in [L]$. In particular, $f$ is a polynomial of degree $\prod_{k=1}^L N_k$.
\end{lemma}
\begin{proof}[Proof]
We can show this by induction over $L$. First, consider the case $L=1$. By Lemma \ref{lemma:jackson}, for each $j \in [d_1]$, there exist polynomials $p_{N,j}:\R\to\R$ of degree $N$ which verify
$$
\abs*{p_{N,j}((\ba^{(1)}_i)^T \bx) - \sigma^{(1)}_j((\ba^{(1)}_j)^T \bx)} \leq \frac{1}{N}
$$
since $\abs*{(\ba^{(1)}_i)^T \bx} \leq 1$ by assumption. Since $\norm{\ba^{(2)}}_1 \leq 1$, it follows that
$$
\abs*{(\ba^{(2)})^T \bp_N(\bA^{(1)}\bx) - (\ba^{(2)})^T\bsigma^{(1)}(\bA^{(1)}\bx)} \leq \frac{1}{N}~.
$$
This implies the thesis for the case $L = 1$. Now consider the induction step, that is, assume that, for every $\delta > 0$ and $j$, there exists a certain $f_{N_1,\dots,N_{L-1}}^j$ such that
$$
\abs*{x^{(L-1)}_j(\bx) - f^j_{N_1,\dots,N_{L-1}}(\bx)} \leq \delta
$$
as long as $N_k \geq \frac{L-1}{\delta} + (L-2)$ for $k \in [L-1]$. Notice that this implies that
$$
\abs*{(\ba^{(L)}_j)^T\bfun_{N_1,\dots,N_{L-1}}(\bx)} \leq 1+\delta~,
$$
where $\bfun_{N_1,\dots,N_{L-1}} = (f^{1}_{N_1,\dots,N_{L-1}},\dots, f^{d_{L-1}}_{N_1,\dots,N_{L-1}})$.
Therefore for each $j \in [d_L]$, by Lemma \ref{lemma:jackson}, there exist polynomials $p_{N,j}$ of degree $N$ such that 
$$
\abs*{p_{N,j}((\ba_j^{(L)})^T\bfun_{N_1,\dots,N_{L-1}}(\bx)) - \sigma^{(L)}_j((\ba_j^{(L)})^T\bfun_{N_1,\dots,N_{L-1}}(\bx))} \leq \frac{1 + \delta}{N}~.
$$
Let then $f_{N_1,\dots,N_{L-1},N}$ be defined as
$$
f_{N_1,\dots,N_{L-1},N}(\bx) = \sum_{j=1}^N a^{(L+1)}_{j}p_{N,j}((\ba_j^{(L)})^T\bfun_{N_1,\dots,N_{L-1}}(\bx))~.
$$
Since $\norm{\ba^{(L+1)}}_1 \leq 1$, it holds that
\begin{align*}
\abs*{ f_{N_1,\dots,N_{L-1},N}(\bx) - f(\bx) } \leq & \abs*{ f_{N_1,\dots,N_{L-1},N}(\bx) - \ba_{L+1}^T\bsigma^{L+1}\parr*{f_{N_1,\dots,N_{L-1},N}(\bx)} } \\
& + \abs*{ \ba_{L+1}^T\bsigma^{L+1}\parr*{f_{N_1,\dots,N_{L-1},N}(\bx)} - f(\bx) } \\
\leq & \, \frac{1+\delta}{N} + \delta~.
\end{align*}
If $\delta = \frac{L-1}{L}\epsilon$ then equation \eqref{eq:poly_multi_layer_lemma} holds as long as 
$$
N \geq \frac{1 + \frac{L-1}{L}\epsilon}{\frac{\epsilon}{L }} = \frac{L}{\epsilon} + (L-1)
~.
$$
This concludes the proof of the lemma.
\end{proof}

\begin{proof}[Proof of Proposition \ref{prop:multi-layer-ub-polyd}]
It holds that 
$$
f(\bx) = g(\bsigma^{(1)}(\bA^{(1)} \bx))
$$
where $g$ is a $(L-1)$-hidden-layers neural network with input dimension $d_1$. By Lemma \ref{lemma:burkill}, for every $\delta > 0$ and $j\in[d_1]$, there exists Fourier networks $q_{N_1,j}(\bx)$ with $2N_1 -1$ units such that
$$
\abs*{\sigma^{(1)}_j((\ba^{(1)}_j)^T\bx) -q_{N_1,j}((\ba^{(1)}_j)^T\bx)} \leq \frac{C}{\sqrt{N_1}}
$$
where $C>0$ is a numerical constant. Notice that this implies that, for $N_1 \geq 4 C^2$, it holds
$$
\norm*{\bq_{N_1}(\bA^{(1)}\bx)}_\infty \leq 1 ~.
$$
Now, we can approximate $g$ with a polynomial neural network $g_{N_L,\dots,N_2}$ as given by Lemma \ref{lemma:poly_upper_bound_multi_layer}. In particular, for any $\delta > 0$, there exist $g_{N_L,\dots,N_2}$ such that
$$
\sup_{\bx\in [-1,1]^d}\abs*{g_{N_L,\dots,N_2}(\bx) - g(\bx)} \leq \delta
$$
as long as $N_k \geq \frac{L-1}{\delta } + (L-2)$ for $k \in [2,L]$. It follows that
\begin{align*}
\abs*{g_{N_L,\dots,N_2}(\bq_{N_1}(\bA^1\bx)) - f(\bx)} \leq \delta + \frac{C}{\sqrt{N_1}}~.
\end{align*}
Let $f_N(\bx) = g_{N_L,\dots,N_2}(\bq_{N_1}(\bA^{(1)}\bx))$. By choosing $\delta = \epsilon / 2$, it holds that
$$
\sup_{\bx\in [-1,1]^d}\abs*{f_N(\bx) - f(\bx)}\leq \epsilon
$$
as long as $N_k \geq 2\frac{L-1}{\epsilon} + (L-2)$  for $k \in [2,L]$ and $N_1 \geq C^2\parr*{1 + \frac{4}{\epsilon^2}}$. We claim that $f_N$ is a Fourier network with at most
\begin{equation}\label{eq:prop_multi_layer_N_in_N_i}
N = \parr*{2^L N_1 d_1}^{\prod_{k=2}^{L}N_k}
\end{equation}
units. We can prove this by induction over $L \geq 2$. Remember that $g_{N_L,\dots,N_2}$ is is the form
$$
g_{N_L,\dots,N_2}(\bx) =\parq*{  \ba^{(L+1)} }^T\bg^L_{N_L}\parr*{ \bA^{(L)}\bg^{L-1}_{N_{L-1}}\parr*{ \bA_{(L-1)} \cdots  \bg^2_{N_2}(\bA^{(2)} \bx)} }
$$
where $\bg^k_{N_k}$ is a component-wise polynomial of degree at most $N_k$, for $k \in [2,L]$.
We start by the case $L=2$. Notice that each component of $\bA^{(2)}\bq_{N_1}(\bA^{(1)} \bx)$ is a Fourier network with the same set of first layer weights, of size at most $(2N-1)d_1$. Then, by Lemma \ref{lemma:Fourier_nets_prod_multi}, we have that each component of 
$$
\bfun^2_{N_2,N_1}(\bx) \doteq \bA^{(3)}\bg_{N_2}^2(\bA^{(2)}\bq_{N_1}(\bA^{(1)}\bx))
$$
is a Fourier network with the same set of first layer weights of size at most $$
(2(2N_1 - 1)d_1)^{N_2} ~.
$$
Finally, consider the induction step. By the assumption hypothesis, the function
$$
\bfun^{L-1}_{N_{L-1},\dots,N_1}(\bx) \doteq \bA^{(L)}\bg_{N_{L-1}}^{L-1}(\bA^{(L-1)} \cdots \bg_{N_2}^2(\bA^{(2)}\bq_{N_1}(\bA^{(1)} \bx)))
$$
is such that each component is a Fourier network with the same set of first layer weights of size at most
$$
\parr*{2^{L-2}(2N_1-1)d_1}^{\prod_{k=2}^{L-1}N_k}~.
$$
Then, by Lemma \ref{lemma:Fourier_nets_prod_multi}, the function 
$$
f_N(\bx) = \parq*{\ba^{(L+1)}}^T\bg_{N_L}^L(\bfun^{L-1}_{N_{L-1},\dots,N_1}(\bx))
$$
is a Fourier network with at most
\begin{align*}
\parr*{2 \cdot \parr*{2^{L-2}(2N_1-1)d_1}^{\prod_{k=2}^{L-1}N_k}}^{N_L} & = 2^{N_L}2^{(L-2)\prod_{k=2}^{L-1}N_k}\parr*{(2N_1-1) d_1}^{\prod_{k=2}^{L}N_k}
\end{align*}
which implies equation \eqref{eq:prop_multi_layer_N_in_N_i}. Plugging in the lower bounds on $N_k$ in terms of $\epsilon$, the thesis follows.
\end{proof}

\subsection{Fixed-dimension approximation}\label{sec:fixed_d_app}

The results of Section \ref{sec:ub} on fixed-threshold approximation can be complemented by the following result on fixed-dimension approximation. The proposition below is a straight-forward generalization of Theorem 3 in \citep{safran2019depth}.

\begin{proposition}
Let $\sigma$ be an activation satisfying Assumption \ref{ass:activation}. Then there exists a constant $\beta >0$ such that for any $f:B^d_{1,2}\to \C$ $1$-Lipschitz function and $\epsilon >0$ there exists a network $f_N \in \mathcal{F}_N^\sigma$ such that
$$
\norm*{f - f_N}_{B_{1,\infty}^d,\infty} \leq \epsilon
$$
for some $N \leq 2 + \beta d^7 \parr*{ \beta \epsilon^{-1}}^d \epsilon^{-6}$.
\end{proposition}
\begin{proof}[Proof]
The result is proved by noticing that the proof of Theorem 3 in \citep{safran2019depth} actually holds for any function $f$ as in the statement. Moreover, using Assumption \ref{ass:activation}, $f_N$ can also be chosen so that an equivalent bound holds for $m_\infty(f_N)$.
\end{proof}

\section{Proofs related to spherical harmonics analysis of shallow networks}

\subsection{Proof of Proposition \ref{prop:non_efficient_sphere}}\label{sec:proof_non_efficient_sphere}

Let $f_N:\R^d \to \R$ a one-hidden-layer network defined by
$$
f_N(\bx) = \sum_{i=1}^N u_i f^{\sigma_i,\bw_i}(\bx) \doteq \sum_{i=1}^N u_i \sigma_i\parr*{\bw_i^T \bx}
$$
where $\bu \in \R^N$, $\bw_i \in \S$, and $\sigma_i$ are linearly bounded activations. 
Thanks to Parseval's formula, it holds that
\begin{align}
\norm{f_N - f^{(d)}}_2^2 & \geq \norm{\mathcal{P}_{I_d}f_N - \mathcal{P}_{I_d} f^{(d)}}_2^2     \nonumber\\
& \geq \norm{\mathcal{P}_{I_d} f^{(d)}}_2^2 - 2\sum_{j\in I_d} \sum_{i=1}^N u_i \prodscal{f^{\sigma_i,\bw_i}, f^{(d)}_{j}} \nonumber\\
& \geq \norm{\mathcal{P}_{I_d} f^{(d)}}_2^2 - 2\sum_{j\in I_d} \sum_{i=1}^N \frac{1}{\sqrt{N^d_j}}\abs*{u_i} \norm{f_j^{(d)}}_\infty\norm{f_j^{\sigma_i,\bw_i}}_2 \label{eq:lb_sh} \\
& \geq \norm{\mathcal{P}_{I_d} f^{(d)}}_2^2 - 2\norm{ f^{(d)}}_2\sum_{i=1}^N\abs*{u_i}\norm{f^{\sigma_i,\bw_i}}_2 \parq*{ \sum_{j\in I_d} c_{d,j}^2 }^{1/2} \nonumber\\ &  \geq \norm{\mathcal{P}_{I_d} f^{(d)}}_2^2 - 2\cdot O(d^M)\cdot \epsilon^{d^\alpha}\norm{ f^{(d)}}_2\sum_{i=1}^N\abs*{u_i}\norm{f^{\sigma_i,\bw_i}}_2 \nonumber ~.
\end{align} 
Finally, notice that it holds that
$$
\norm{f^{\sigma_i,\bw_i}}_2 \leq 2 \,m_\infty(f_N)
$$
and therefore
\begin{align*}
\norm{f_N - f^{(d)}}_2^2 & \geq \Omega( d^{-2M} ) - 4 \cdot O( d^M )\cdot \epsilon^{d^\alpha}\cdot m_\infty^2(f_N) \cdot N~.
\end{align*}
This concludes the proof.

\subsection{Low-coherence zonal harmonics frames}\label{sec:low_coherence_frames}

In this section, we wish to quantify how much incoherent can a frame composed of zonal harmonics be. More specifically, we wish to find a lower bound for 
$$
N(d,k,\epsilon) = \sup\bra*{N \geq 1 \st \exists\; \bw_1,\dots,\bw_N\in \S \st \sup_{i\neq j}\abs*{P_k^d\parr*{\bw_i^T\bw_j}} \leq \epsilon}
$$
for $\epsilon \in (0,1)$.
\begin{lemma}
It holds that
$$
N(d,k,\epsilon) \geq \sup\bra*{N \geq 1 \st \exists\; \bw_1,\dots,\bw_N\in \S \st \sup_{i\neq j}\abs*{\bw_i^T\bw_j} \leq \sqrt{1-\frac{d}{k\epsilon^{4/d}}}}
$$
for $k> d\geq 5$ and $\parr*{\frac{d}{k}}^{d/4} \leq \epsilon  < 1$. 
\end{lemma}
\begin{proof}[Proof]
We recall that it holds
$$
\abs*{P_k^d(t)} \leq \frac{1}{\sqrt{\pi}}\Gamma\parr*{\frac{d-1}{2}}\parr*{\frac{4}{k(1-t^2)}}^{(d-2)/2}
$$
for $d\geq 2$ and $t\in(-1,1)$ (cfr. eq. (2.117) in \citep{atkinson2012spherical}) and that
$$
\Gamma(x) \leq \parr*{\frac{x}{2}}^{x-1}
$$
for $x \geq 2$. Therefore it holds that
\begin{align*}
\abs*{P_k^d(t)} & \leq \frac{1}{\sqrt{\pi}}\parr*{\frac{d-1}{4}}^{(d-3)/2}\parr*{\frac{4}{k(1-t^2)}}^{(d-2)/2} \\
& \leq \frac{1}{\sqrt{\pi}}\parr*{\frac{d}{4}}^{-1/2}\parr*{\frac{d}{k(1-t^2)}}^{(d-2)/2} \leq  \parr*{\frac{d}{k(1-t^2)}}^{(d-2)/2}
\end{align*}
for $d \geq 5$ and $\abs{t} < 1$. In particular, for $\epsilon \in (0,1)$, it holds that $\abs*{P_k^d(t)} \leq \epsilon$ if 
$$
\frac{d}{k(1-t^2)} \leq \epsilon^{4/d}
$$
that is if 
$$
\abs*{t} \leq \sqrt{1-\frac{d}{k\epsilon^{4/d}}}~.
$$
The thesis follows.
\end{proof}
Define 
$$
N(d,\delta ) = \sup\bra*{N \geq 1 \st \exists\; \bw_1,\dots,\bw_N\in \S \st \sup_{i\neq j}\abs*{\bw_i^T\bw_j} \leq \delta}
$$
for $\delta \in (0,1)$. The previous lemma says that 
$$
N(d,k,\epsilon) \geq N\parr*{d,\sqrt{1-\frac{d}{k\epsilon^{4/d}}}}~.
$$
\begin{example}
Taking 
\begin{equation}\label{eq:rad_vectors_sign}
\bra*{\bw_i}_{i=1}^N = \bra*{ \epsilon \in \bra*{\pm\frac{1}{\sqrt{d}}}^d \st \epsilon_1 >0 }
\end{equation}
it holds that $N = 2^{d-1}$ and 
$$
\max_{i\neq j} \abs*{\bw_i^T \bw_j} = 1 - \frac{2}{d}~.
$$
Therefore 
$$
N\parr*{d,1 - \frac{2}{d}} \geq 2^{d-1}~.
$$
Taking $\epsilon = 2^{-d}$, it holds that, if $k \geq 8d^2$, then 
$$
N\parr*{d,k,2^{-d}} \geq 2^{d-1}~.
$$
\end{example}

Using this fact it is possible to explicitly construct a high energy sparse function.
\begin{lemma}\label{lemma:high_energy_sparse}
Take $k \geq 16d^2$ even and let
$$
\hat{P}(\bx) = \beta_d \sum_{i=1}^{2^{d-1}} (N_k^d)^{1/2}P_k^d\parr{\bw_i^T\bx}
$$
with $\beta_d = 2(2^d + 2)^{-1/2}$ and $\bw_i$ as in equation \eqref{eq:rad_vectors_sign}. Then $\norm{\hat{P}}_2 = \Theta_d(1) $ and it is exponentially spread, that is $\ell_{\infty,2}(\hat{P}) \leq O_d(2^{-d/2})\sqrt{N_k^d}$.
\end{lemma}
\begin{proof}[Proof]
It holds that
\begin{align*}
\norm{\hat{P}}_2^2 & = \beta_d^2\parq*{ 2^{d-1} + \sum_{i\neq j} P_k^d\parr*{\bw_i^T \bw_j} } \\ &  \leq \frac{2}{2^{d-1} + 1}\parq*{ 2^{d-1} + \parr*{2^{2d-2} - 2^{d-1}}2^{-d} }
\\ & = \frac{2}{2^{d-1} + 1}\parq*{ 2^{d-1} + 2^{d-2} - 2^{-1} } \leq 3
\end{align*}
and that 
\begin{align*}
\norm{\hat{P}}_2^2 & \geq \frac{2}{2^{d-1} + 1}\parq*{ 2^{d-1} - \parr*{2^{2d-2} - 2^{d-1}}2^{-d} }
\\ & = \frac{2}{2^{d-1} + 1}\parq*{ 2^{d-1}  -2^{d-2} + 2^{-1} } \geq 1~.
\end{align*}
On the other hand, it holds that
\begin{align*}
\norm{\hat{P}}_\infty &  \leq \beta_d (N_k^d)^{1/2} \sup_{x\in\S} \sum_{i=1}^{2^{d-1}}\abs*{ P_k^d(\bw_i^T\bx) }~.
\end{align*} 
By definition of the vectors $\bra*{\bw_i}_{i=1}^{2^{d-1}}$, it holds
\begin{align*}
\sup_{x\in\S} \sum_{i=1}^{2^{d-1}}\abs*{ P_k^d(\bw_i^T\bx) } & = \frac{1}{2}\sup_{x\in\S,\, x > 0}\sum_{\bepsilon \in \bra{\pm d^{-1/2}}^d} \abs*{P_k^d( 
\bx^T\bepsilon )} \\
& \leq 1 + \frac{1}{2}\sup_{\bx\in\S,\, \bx \succ 0}\sum_{\bepsilon \in \bra{\pm d^{-1/2}}^d 
\st \abs{\mathbf{1}^T\bepsilon} < \sqrt{d} } \parr*{ \frac{1}{16d \parr*{ 1- \abs*{\bx^T\bepsilon}^2} } }^{(d-2)/2}
\\
& \leq 1 + \frac{1}{2} (2^d-2) \parr*{ \frac{1}{16d \parr*{ 1- \frac{d-1}{d}} } }^{(d-2)/2} \leq 1 + \frac{2^{d-1} - 1}{4^{d-2}} \leq 2 ~.
\end{align*} 
This proves the claim.
\end{proof}

\subsection{Proof of Proposition \ref{prop:F1-space}}\label{sec:proof_F1-space}

Assume first that $f\in\mathcal{H}^1$. Then $f = h_\pi$ for some $\pi$ even signed Radon measure. Thus
\begin{align*}
\gamma_1(f) = \norm{\pi}_{1} & = \sup_{\varphi \in C(\S) \st \norm{\varphi}_\infty \leq 1} \int_\S \varphi(\bw) \,d\pi(\bw) \\
& = \sup_{\varphi \in C^\infty_{even}(\S) \st \norm{\varphi}_\infty \leq 1} \int_\S \varphi(\bw)\,d\pi(\bw)
\\
& = \sup_{\varphi \in C^\infty_{even}(\S) \st \norm{\varphi}_\infty \leq 1} \int_\S T(T^{-1} \varphi)(\bw)\,d\pi(\bw) \\
& = \sup_{\varphi \in C^\infty_{even}(\S) \st \norm{\varphi}_\infty \leq 1}  \int_\S \int_\S \abs*{\bw^T\bx}(T^{-1} \varphi)(\bx) \,dS(\bx)\,d\pi(\bw) 
\\
& = \sup_{\varphi \in C^\infty_{even}(\S) \st \norm{\varphi}_\infty \leq 1}  \prodscal{T^{-1}\varphi, f} ~.
\end{align*}
This shows one side of the statement. On the other hand, assume that 
$$
\sup_{\varphi\in C^\infty_{even}(\S)\st \norm{\varphi}_\infty \leq 1} \prodscal{T^{-1}\varphi, f} < \infty~.
$$
Then, the transformation
$$
S_f(\varphi) \doteq   \prodscal{T^{-1}\varphi,f}
$$
defines a bounded linear operator $S_f:C^\infty_{even}\to\R$. Since $C^\infty_{even}(\S)$ is dense in $C_{even}(\S)$ (the set of even function in $C(\S)$), $S_f$ can be extended to a bounded linear operator on $C_{even}(\S)$. By setting 
$$
S_f(\varphi) = S_f(\varphi_{even})
$$
we can extend it on $C(\S)$. By the Riesz representation theorem, there exists a signed Radon measure $\pi$ on $\S$ such that 
$$
S_f(\varphi) = \int_\S \varphi(\bw)\,d\pi(\bw)
$$
for every $\varphi\in C(\S)$. Moreover, since $S_f(\varphi) = 0$ for every odd $\varphi$, we can assume that $\pi$ is even. Let $h_\pi$ be the function in $\mathcal{H}^1$ defined by $\pi$. Then it holds that
$$
\prodscal{T^{-1}\varphi, f} = \norm{\pi}_{1} = \prodscal{T^{-1}\varphi, h_\pi}
$$
for every $\varphi \in C^\infty_{even}(\S)$. Since $T$ is an automorphism over $C^\infty_{even}(\S)$, then it holds 
$$
\prodscal{\varphi, f} =  \prodscal{\varphi, h_\pi}
$$
for every $\varphi \in C^\infty_{even}(\S)$. Since $f$ and $h_\pi$ are even, this implies that $f = h_\pi$. This concludes the proof.

\end{document}